\documentclass{article}[11pt]

\usepackage{amsmath,amsthm,amssymb}
\usepackage{fullpage}
\usepackage{mathtools,bbm,xspace}
\usepackage{mathrsfs}
\usepackage{bm}

\usepackage[utf8]{inputenc} % allow utf-8 input
\usepackage[T1]{fontenc}    % use 8-bit T1 fonts
\usepackage{url}            % simple URL typesetting
\usepackage{booktabs}       % professional-quality tables
\usepackage{amsfonts}       % blackboard math symbols
\usepackage{nicefrac}       % compact symbols for 1/2, etc.
\usepackage{microtype}      % microtypography
\usepackage{algorithm}
\usepackage{algpseudocode}
\usepackage[dvipsnames]{xcolor}
\usepackage{subfig}
\usepackage{ifxetex}

\usepackage{verbatim}
\usepackage[
letterpaper,
top=1in,
bottom=1in,
left=1in,
right=1in]{geometry}

\ifxetex
\usepackage[bigdelims,vvarbb]{newpxmath}
\usepackage{mathspec}
\setallmainfonts(Digits,Latin){Noto Serif}
\else
\usepackage{mathpazo}
\fi

\usepackage[pagebackref]{hyperref}
\hypersetup{
pagebackref,
letterpaper=true,
colorlinks=true,
urlcolor=blue,
linkcolor=blue,
citecolor=OliveGreen,
}
\usepackage[capitalise,nameinlink]{cleveref}
\crefname{lemma}{Lemma}{Lemmas}
\crefname{fact}{Fact}{Facts}
\crefname{theorem}{Theorem}{Theorems}
\crefname{corollary}{Corollary}{Corollaries}
\crefname{claim}{Claim}{Claims}
\crefname{example}{Example}{Examples}
\crefname{problem}{Problem}{Problems}
\crefname{definition}{Definition}{Definitions}
\crefname{assumption}{Assumption}{Assumptions}
\crefname{subsection}{Subsection}{Subsections}
\crefname{section}{Section}{Sections}

\newtheorem{theorem}{Theorem}[section]

\newtheorem{proposition}[theorem]{Proposition}
\newtheorem{lemma}[theorem]{Lemma}
\newtheorem{corollary}[theorem]{Corollary}

\theoremstyle{definition}
\newtheorem{definition}[theorem]{Definition}

\newtheorem{problem}[theorem]{Problem}

\newtheorem{exercise-easy}[theorem]{Exercise}
\newtheorem{exercise-med}[theorem]{Exercise}
\newtheorem{exercise-hard}[theorem]{Exercise$^\star$}
\newtheorem{claim}[theorem]{Claim}

\allowdisplaybreaks

\DeclareMathOperator*{\argmin}{arg\,min}

\DeclarePairedDelimiterX{\norm}[1]{\lVert}{\rVert}{#1}

\DeclarePairedDelimiterX{\abs}[1]{\lvert}{\rvert}{#1}
\DeclarePairedDelimiterX{\inp}[2]{\langle}{\rangle}{#1, #2}

\DeclarePairedDelimiterX{\infdivx}[2]{(}{)}{%
  #1\;\delimsize\|\;#2%
}

\newcommand{\mc}[1]{\mathcal{#1}}
\newcommand{\mb}[1]{\mathbb{#1}}
\newcommand{\mrm}[1]{\mathrm{#1}}

\newcommand{\R}{\mathbb{R}}
\newcommand{\B}{\mathbb{B}}
\newcommand{\N}{\mathbb{N}}
\renewcommand{\S}{\mathbb{S}}

\renewcommand{\P}{\mathbb{P}}
\newcommand{\E}{\mathbb{E}}

\newcommand{\lprp}[1]{\left(#1\right)}
\newcommand{\lbrb}[1]{\left\{#1\right\}}
\newcommand{\lsrs}[1]{\left[#1\right]}
\newcommand{\wt}[1]{\widetilde{#1}}
\newcommand{\wh}[1]{\widehat{#1}}

\newcommand{\eps}{\varepsilon}

\newcommand{\Ot}{\widetilde{O}}

\newcommand{\median}{\mathrm{Median}}
\newcommand{\diam}{\mathrm{Diam}}
\newcommand{\unif}{\mathrm{Unif}}
\newcommand{\supp}{\mathrm{Supp}}

\def\authornotes{1pt}

\ifdim\authornotes=1pt
\newcommand{\ynote}[1]{\footnote{\color{ForestGreen}Yeshwanth: #1}}
\newcommand{\jnote}[1]{\footnote{\color{Orange}Jelani: #1}}
\else
\newcommand{\ynote}[1]{}
\newcommand{\jnote}[1]{}
\fi

\usepackage{tikz}

\newcommand{\kernEmbedRep}{\widetilde{\Omega} \left( \eps^{-2} \diam (\mathcal{W})^2 \log 1 / \delta \right)}

\newcommand{\lipEmbedRep}{C \eps^{-2} \log^5 (d / \eps) \log 1 / \delta}

\newcommand{\lipRes}{\left(\frac{\eps}{10d}\right)^3}
\newcommand{\lipWid}{\left(10 \sqrt{\log (d / \eps)}\right)^6}
\newcommand{\lipWidLoss}{\left(10 \sqrt{\log (d / \eps)}\right)^3}
\newcommand{\lipSens}{\left(\frac{\eps}{10d}\right)^3}
\newcommand{\lipAcc}{\left(\frac{\eps}{256 \log (d / \eps)}\right)}
\newcommand{\lipNSplits}{32 \log (d / \eps)}

\newcommand{\prodDEDS}{Produce Distance Estimation Data Structure}
\newcommand{\prodDE}{Produce Distance Estimates}
\newcommand{\updateDE}{Update Distance Estimation Data Structure}
\newcommand{\distEstRep}{\wt{\Omega} (\eps^{-2} \log 1 / \delta)}
\newcommand{\distEstSamps}{\wt{\Omega} (\eps^{-2} \log 1 / \delta)}

\DeclareMathOperator{\quant}{Quant}
\DeclareMathOperator*{\diag}{diag}
\DeclareMathOperator*{\Diag}{Diag}

\begin{document}

\title{Uniform Approximations for Randomized Hadamard Transforms with Applications}
\author{Yeshwanth Cherapanamjeri\thanks{UC Berkeley. \texttt{yeshwanth@berkeley.edu}. Supported by a Microsoft Research BAIR Commons Research Grant} \and Jelani Nelson\thanks{UC Berkeley. \texttt{minilek@berkeley.edu}. Supported by NSF award CCF-1951384, ONR grant N00014-18-1-2562, ONR DORECG award N00014-17-1-2127, and a Google Faculty Research Award.}}
\date{}
\maketitle

\begin{abstract}

  Randomized Hadamard Transforms (RHTs) have emerged as a computationally efficient alternative to the use of dense unstructured random matrices across a range of domains in computer science and machine learning. For several applications such as dimensionality reduction and compressed sensing, the theoretical guarantees for methods based on RHTs are comparable to approaches using dense random matrices with i.i.d.\ entries. However, several such applications are in the low-dimensional regime where the number of rows sampled from the matrix is rather small. Prior arguments are not applicable to the high-dimensional regime often found in machine learning applications like kernel approximation. Given an ensemble of RHTs with Gaussian diagonals, $\{M^i\}_{i = 1}^m$, and any $1$-Lipschitz function, $f: \R \to \R$, we prove that the average of $f$ over the entries of $\{M^i v\}_{i = 1}^m$ converges to its expectation uniformly over $\norm{v} \leq 1$ at a rate comparable to that obtained from using truly Gaussian matrices. We use our inequality to then derive improved guarantees for two applications in the high-dimensional regime: 1) kernel approximation and 2) distance estimation. For kernel approximation, we prove the first \emph{uniform} approximation guarantees for random features \cite{rechtrahimi} constructed through RHTs lending theoretical justification to their empirical success \cite{fastfood,orthogonal_random_features} while for distance estimation, our convergence result implies data structures with improved runtime guarantees over previous work by the authors. We believe our general inequality is likely to find use in other applications.
 \end{abstract}

 \thispagestyle{empty}
\setcounter{page}{0}
\newpage

\section{Introduction}
\label{sec:intro}

% Intro outline:
% \begin{enumerate}
%     \item Randomized linear transforms are amazing used everywhere! 
%     \item Gaussians are simplest known
%     \item Used everywhere!
%     \item RHTs emerged as alternative 
%     \item Define RHTs
%     \item Discuss RHTs and why they're not independent
%     \item Say known guarantees known in under-parametrized case
%     \item Bring up kernel approximation
%     \item Introduce theorem
%     \item Discuss why theorem is significant and why its hard -- standard gridding doesn't work and it isn't clear how much independence is there
%     \item Discuss Kernel Approximation in more detail
%     \item Give Kernel approximation theorem
%     \item Discuss Distance estimation 
%     \item Give Distance estimation theorem
% \end{enumerate}

Randomized linear mappings find ubiquitous application in diverse domains across computer science and machine learning. Representing a linear transformation $f: \R^d \to \R^k$ as a matrix $\Pi \in \R^{k \times d}$ such that $f(x) = \Pi x$, a commonly examined randomized linear mapping is one where the entries of $\Pi$ are drawn i.i.d.\ from a simple distribution; say, a standard normal. Randomized matrices of the previous form have found use as tools for compressed sensing \cite{CandesT05,Donoho04}, dimensionality reduction \cite{im98}, machine learning \cite{rechtrahimi}, and differential privacy \cite{BlockiBDS12}, amongst other areas. However, one downside to the use of such transformations is that they can be slow, as applying the map amounts to dense matrix-vector multiplication.

Randomized Hadamard Transforms (RHTs) have emerged as a versatile alternative to the use of fully random matrices in applications ranging from the construction of fast Johnson-Lindenstrauss transforms \cite{fastjl}, to speeding up iterative recovery methods in compressed sensing \cite{CandesT06,NeedellT09}, designing fast algorithms for approximate regression and low-rank approximation \cite{Sarlos06}, and building faster algorithms for deep learning \cite{performers}; their special structure allowing for faster computation of the mapping. Assuming $d = 2^\ell$ for some $\ell \in \N$, the RHT is defined as follows:
\begin{equation*}
    f(x) = H_d D x \text{ where } H_d = 
    \begin{bmatrix}
        H_{d / 2} & H_{d / 2} \\
        H_{d / 2} & -H_{d / 2}
    \end{bmatrix}
    \text{ with }
    H_1 = [1] \text{ and }
    D_{i,j} \thicksim
    \begin{cases}
        \mc{N}(0, 1) & \text{if } i = j\\
        0 &\text{otherwise}
    \end{cases}.
\end{equation*}
Due to the recursive structure of the Hadamard matrix, the mapping $f(x)$ can be computed in time $O(d\log d)$ as opposed to $O(d^2)$ for a matrix with i.i.d.\ entries. While each row of the matrix is distributed as a standard normal vector, entires in a column are correlated due to the shared randomness from $D$. Despite these correlations, in ``low-dimensional'' applications such as dimensionality reduction and compressed sensing, where a small number of rows are sampled from $f(x)$, prior work has shown that (subsampled) RHTs provide guarantees competitive with the use of Gaussian random matrices. 

However, for ``high-dimensional'' applications frequently found in machine learning where $k$ may be significantly larger than $d$, known guarantees for RHTs are not comparable to those for Gaussian random matrices. As a concrete example, consider the problem of approximating the RBF kernel, defined as $K_{\mrm{RBF}}(x,y) \coloneqq \exp \lbrb{- \frac{\norm{x - y}^2}{2}}$ where $\norm{\cdot}$ denotes the Euclidean norm. In their pioneering work, Rahimi and Recht \cite{rechtrahimi} construct an embedding, $h_g (x)$, of dimension $(d / \eps^2)$ based on Gaussian random matrices such that $\abs{\inp{h_g (x)}{h_g (y)} - K_{\mrm{RBF}}(x,y)} \leq \eps$ for all $x,y$ in a bounded ball. These embeddings have grown to become one of the most widely adopted techniques for scaling up kernel methods and as such have been hugely influential in machine learning, with its impact recognized in NeurIPS 2017 with a Test of Time Award. Due to their widespread use, much effort has been devoted toward improving the computational complexity of these methods with approaches based on RHTs emerging as a popular alternative with comparable empirical performance and significantly faster runtimes \cite{fastfood,orthogonal_random_features}. However, in sharp contrast to the situation for Gaussian matrices, there are no known uniform concentration results for methods based on RHTs, despite their superior computational properties and empirical performance \cite{orthogonal_random_features,unreasonable_effectiveness}. 

% Our main result is a uniform concentration inequality on RHTs with the goal of bridging the gap between RHTs and full Gaussian matrices in the ``high-dimensional'' setting where we show that the empirical distribution of output of a RHT converges to a Gaussian at a rate comparable to that obtained for full Gaussian matrices. We use our result to establish improved theoretical guarantees for two ``high-dimensional'' problems: kernel approximation and distance estimation, illustrating its broad applicability. We introduce some notation then state our main result. Below, $\mc N(0,1)$ denotes the Gaussian with mean $0$ and variance $1$.

Our main result is a uniform concentration inequality on RHTs with the goal of bridging the gap between RHTs and full Gaussian matrices in the ``high-dimensional'' setting where we show that for any Lipschitz function, its average over the entries of the output of a RHT converges uniformly (over inputs to the RHT) to its expectation at a rate comparable to that obtained for full Gaussian matrices. We use our result to establish improved theoretical guarantees for two ``high-dimensional'' problems: kernel approximation and distance estimation, illustrating its broad applicability. We introduce some notation then state our main result as \cref{thm:lip_conc}. Below and in the rest of the paper, $\mc N(0,\sigma^2)$ denotes the Gaussian with mean $0$ and variance $\sigma^2$.
\begin{gather*}
    \{D^j\}_{j = 1}^m \in \R^{d\times d} \text{ s.t } (D^j)_{k, l} \overset{iid}{\thicksim} 
    \begin{cases}
        \mc{N} (0, 1), &\text{if } k = l \\
        0, &\text{otherwise}
    \end{cases} \\
    \forall z \in \R^d: \wt{h} (z) \coloneqq 
    \begin{bmatrix}
        HD^1 z \\
        HD^2 z \\
        \vdots \\
        HD^m z
    \end{bmatrix}, 
    \quad
    \wt{h}_{j,k} (z) = (HD^jz)_k 
    \text{ and }
    \wt{h}^j (z)= HD^j z \tag{RHT} \label{eq:RHT}
\end{gather*}
\begin{theorem}
    \label{thm:lip_conc}
    Let $d \in \mb{N}, \delta, \eps \in (0, 1/2)$ and $f: \R \to \R$ be a $1$-Lipschitz function. Then we have with probability at least $1 - \delta$:
    \begin{equation*}
        \forall z \in \R^d \text{ s.t } \norm{z} \leq 1: \abs*{\frac{1}{md}\cdot \sum_{j = 1}^m \sum_{k = 1}^d f(\wt{h}_{j,k} (z)) - \E_{Z \thicksim \mc{N} (0, \norm{z}^2)} [f(z)]} \leq \eps
    \end{equation*}
    as long as $m \geq \lipEmbedRep$ for some absolute constant $C > 0$.
\end{theorem}
% \begin{theorem}
%     \label{thm:cdf_conc}
%     Let $d \in \mb{N}, \delta, \eps \in (0, 1/2)$. Then we have with probability at least $1 - \delta$:
%     \begin{equation*}
%         \forall z \in \R^d, t \in \R: \abs*{\frac{1}{md}\cdot \sum_{j = 1}^m \sum_{k = 1}^d \bm{1} \lbrb{\wt{h}_{j,k} (z) \leq t} - \P_{Z \thicksim \mc{N} (0, \norm{z}^2)} (Z \leq t)} \leq \eps
%     \end{equation*}
%     as long as $m \geq \cdfEmbedRep$ for some absolute constant $C > 0$.
% \end{theorem}
We pause to make some remarks regarding \cref{thm:lip_conc}. First, note that the number of rows in the linear transformation is $md$, and hence the concentration properties obtained in \cref{thm:lip_conc} are within a small logarithmic factor of those obtained from the use of full Gaussian matrices where $\approx d / \eps^2$ suffice (see \cref{sec:proof_overview} for a standard proof). Secondly, similar results are \emph{not} obtainable when an alternative distribution, $\mc{D}$, is used in place of Gaussians in the diagonal matrices in the definition of the RHT. To see this, consider the case where the $\mc{D}$ is symmetric and observe that the empirical distribution of the entries of $\wt{h} (e_1)$ converge to $\mc{D}$ while the entries of $\wt{h}(\bm{1})$ converge to a Gaussian as a consequence of the central limit theorem. We conclude our discussion with two complementary lower bounds establishing the tightness of \cref{thm:lip_conc}. In \cref{thm:m_lb}, we show that the constraint on the embedding dimension in terms of $m$ is optimal up to log factors by exhibiting a candidate $1$-Lipschitz function requiring $m \geq \eps^{-2} \log d / \delta$ and finally, in \cref{thm:gau_lb}, we show that there exists a $1$-Lipschitz function requiring an embedding dimension of at least $\eps^{-2} d$ for uniform concentration when true Gaussian random matrices are used. Taken together, these results imply that RHTs are \emph{optimally} comparable (in the sense of \cref{thm:lip_conc}) to random Gaussian matrices and that this phenomenon is \emph{not} an artifact of the looseness of either analysis.

To our knowledge, this is the first uniform concentration inequality of this type for RHTs and we anticipate its use beyond the applications illustrated in our work. We now discuss applications of \cref{thm:lip_conc} to two tasks featuring high dimensional embeddings: kernel approximation and distance estimation. For both of these applications, our result yields significant runtime improvements over prior work.

\subsection{Kernel Approximation}
\label{ssec:kern_appx_intro}
Kernel functions drastically increase the ability of machine learning based methods to learn complex functions of the underlying data. Roughly speaking, these techniques allow the use of a user specified ``inner-product'' function, $K(x,y)$, which corresponds to the inner product $\inp{\phi (x)}{\phi (y)}$ for some function $\phi : \R^d \to \mc{H}$ and Hilbert space $\mc{H}$. For instance, consider a simple classification task where the input data consists of pairs, $\lbrb{(x_i,y_i)}_{i = 1}^n \subset \R^d \times \{0, 1\}$, and the goal is learn a classifier predicting the label $y$ on a new input $x$. Kernel methods represent the classifier as a linear combination $q (x) = \sum_{i = 1}^n \alpha_i K(x_i, x) $ where the coefficients $\alpha_i$ are learnt from data. By parameterizing the classifier in this way, Kernel methods can exploit the flexibility offered by the use of high dimensional embeddings without explicitly performing the embedding which may be computationally expensive/infeasible depending on the kernel used. 

One major drawback of kernel functions is that naively evaluating the classifier on even a single input point can potentially incur a runtime of $nd$. One approach to improve this runtime is the use of Random Fourier Features \cite{rechtrahimi}, in which one embeds the data points into a Euclidean space such that inner products of the embeddings roughly correspond to the evaluation of the kernel. For the popular RBF kernel, their embedding is defined as follows where the function $\cos (\cdot)$ is applied elementwise:
\begin{equation*}
    h(x) = \sqrt{\frac{2}{md}} \cdot \cos (\Pi x + b) \text{ where } \Pi \in \R^{k \times d}, b \in \R^{k} \text{ with } \Pi_{i,j} \overset{iid}{\thicksim} \mc{N} (0, 1),\ b_i \overset{iid}{\thicksim} \unif ([0, 2\pi])
\end{equation*}
For fixed $B > 0$ and $k \approx d / \eps^2$, Recht and Rahimi \cite{rechtrahimi} establish the following claim with high probability:
\begin{equation*}
    \forall x, y \in \B (0, B) : \abs{\inp{h(x)}{h(y)} - K_{\mrm{RBF}} (x,y)} \leq \eps.
\end{equation*}
However, there are no proven universal approximation results when the random transformation $\Pi x$ is replaced by a \ref{eq:RHT} despite their empirical success across a range of machine learning applications \cite{orthogonal_random_features,unreasonable_effectiveness,online_learning_with_orthogonal,unifying_orthogonal,performers}. With this context in mind, we present our theorem for the approximation of the RBF kernel:
\begin{theorem}
    \label{thm:kern_approx}
    Let $d \in \N$, $\delta, \eps \in (0, 1 / 2)$ and $\mc{W} \subset \R^d$ be arbitrary. Then, defining:
    \begin{equation*}
        h(x) = \sqrt{\frac{2}{md}} \cdot \cos (\wt{h} (x) + b)
    \end{equation*}
    where $\wt{h} (\cdot)$ is defined in \ref{eq:RHT} and $b_i \overset{iid}{\thicksim} \unif ([0, 2\pi])$, we have:
    \begin{equation*}
        \forall x, y \in \mc{W}: \abs*{\inp{h(x)}{h(y)} - K_{\mrm{RBF}} (x, y)} \leq \eps
    \end{equation*}
    with probability at least $1 - \delta$ as long as $m \geq \kernEmbedRep$.
\end{theorem}
While previous approaches have shown approximation results for \emph{fixed} $(x,y)$ in expectation \cite[Theorem 1]{orthogonal_random_features}, \cref{thm:kern_approx} is the first \emph{uniform} approximation guarantee for RHTs thus providing theoretical justification for their empirical success. While in-expectation guarantees suffice if a classifier has already been trained and test vectors are chosen independently of the classifier, these approximations are often used in tandem with an iterative optimization procedure during training and in deployment, may face test points which are potentially correlated with predictions on previous inputs. In both these scenarios featuring potentially adaptive inputs, in-expectation guarantees break down while uniform guarantees continue to hold. Note that standard approaches such as generating a new random embedding for each step of an optimization procedure or each input query fail as the coefficients of a linear method utilizing these embeddings are \emph{specific} to the embedding and are unlikely to transfer to a new randomly chosen one. While the dependence of the embedding dimension on $\diam(\mc{W})$ are weaker than those obtained for full Gaussian matrices which have logarithmic dependence on $\diam (\mc{W})$, note that the most interesting regime is when $\norm{x - y} \approx \Ot (1)$ as the RBF kernel decays rapidly in $\norm{x - y}$. 

\subsection{Distance Estimation}
\label{ssec:dist_est_intro}

The second application of our result is in the construction of adaptive algorithms for distance estimation. Formally, the distance estimation problem is defined as follows:
\begin{problem}[Distance Estimation]
    \label{prob:dist_est}
    For a known metric, $d(\cdot, \cdot)$, we are given $X = \{x_i\}_{i = 1}^n \subset \R^d$ and $\eps \in (0, 1)$ and we are required to construct a data structure $\mc{D}$, which when given input query $q$, outputs distance estimates $\lbrb{d_i}_{i = 1}^n$ satisfying:
    \begin{equation*}
        (1 - \eps) d(q,x_i) \leq d_i \leq (1 + \eps) d(q,x_i).
    \end{equation*}
\end{problem}
Our goal will be to build a data structure for distance estimation in the adaptive setting where the sequence of queries seen by the data structure are potentially adversarially chosen with knowledge of answers to previous queries and even potentially the instantiation of the data structure. Note however, that the query cannot depend on future randomness that the algorithm may draw in the process of answering it. The construction of adaptive data structures has received much attention in the recent literature \cite{BenEliezerJWY20,adaptiveds,adv_rob_streaming,diffpriv,adversariallln,separating_adv}. For the particular problem of distance estimation, the approach devised in \cite{adaptiveds} achieves nearly optimal space complexity and query time for all $\ell_p$ ``norms'' for $p \in (0, 2)$.

We now focus solely on the Euclidean setting where our results apply and briefly recall the construction from \cite{adaptiveds}. The approach first draws $l$ i.i.d random Gaussian matrices $\{\Pi_i \in \R^{k \times d}\}_{i = 1}^l$ with $l \approx d$ and $k \approx \eps^{-2}$. For each $x_j \in X$, its embedding, $\Pi_i x_j$ is computed for each $\Pi_i$ and stored. When given query, $q$, the data structure samples $p = O(\log n)$ random matrices, $\{\Pi_{i_r}\}_{r = 1}^p$ and outputs $d_i = \median (\lbrb{\norm{\Pi_{i_r} (q - x_i)}}_{r = 1}^p)$. This approach yields nearly optimal querytimes of $\Ot (n / \eps^2)$. Unfortunately, the update and construction times of the data structure are slow. As the matrices, $\lbrb{\Pi_i}_{i = 1}^l$, have no special structure, adding a new point to the data structure takes time $\Ot (d^2)$ and despite the existence of fast methods for matrix multiplication, construction the data structure is slow ($O(d^{\omega})$ for $n = d$ where $\omega$ is the matrix multiplication constant). 

Our result for distance estimation constructs an algorithm for distance estimation in Euclidean space:

\begin{theorem}
    \label{thm:de_thm}
    Let $\eps, \delta \in (0, 1/2)$. Then, there is a data structure for Distance Estimation in Euclidean space which is initialized correctly with probability at least $1 - \delta$ and supports the following operations:
    \begin{enumerate}
        \item Output a correct answer to a possibly adaptively chosen distance estimation query with probability at least $1 - \delta$
        \item Add input $x \in \R^d$ to the database, $X$.
    \end{enumerate}
    Furthermore, the query and update times of the data structure are $\Ot (\eps^{-2} (n + d) \log 1 / \delta)$ and $\Ot (\eps^{-2} d \log 1 / \delta)$ respectively while the data structure is constructed in time $\Ot (\eps^{-2} (nd) \log 1 / \delta)$.
\end{theorem}

In comparison to \cite[Theorem 4.1]{adaptiveds}, \cref{thm:de_thm} implies a factor $d$ improvement in update and construction times which is significant in high-dimensional applications. Furthermore, \cref{thm:de_thm} yields nearly optimal guarantees as the time taken to construct the data structure is near linear in the size of the data structure and the space complexity was also shown to be optimal in \cite{adaptiveds}.

\paragraph{Organization:} The rest of the paper is organized as follows. We give a brief overview of the proof of \cref{thm:lip_conc} in \cref{sec:proof_overview}. We then present the formal proof in \cref{sec:lip_proof} and describe applications to kernel approximation in \cref{sec:kern_approx_proof} where we prove \cref{thm:kern_approx} and distance estimation in \cref{sec:dist_est} which proves \cref{thm:de_thm}. Finally, \cref{sec:optimality} contains proofs of our lower bounds establishing the optimality of \cref{thm:lip_conc} while \cref{sec:misc} contains standard inequalities and basic technical results used in our proofs.

\paragraph{Notation:} Throughout the paper, $\wt{h}(\cdot)$ denotes the RHT defined in \ref{eq:RHT}. For $x \in \R^d$, we use $\B(x, r)$ to denote the ball of radius $r$ around $x$, $\norm{x}$ and $\norm{x}_\infty$ to denote its Euclidean and infinity norm and $\norm{x}_0$ and $\supp (x)$ will denote the size of its support and its support respectively. When used with a matrix $M \in \R^{p \times q}$, $\norm{M}$ will denote the spectral norm of $M$. For $x \in \R^d$ and a diagonal matrix $B \in R^{d \times d}$, we use $\Diag (x)$ to denote the diagonal matrix, $D$, with the entries of $x$ along the diagonal while $\diag (B)$ denotes the vector consisting of the diagonal entries of $B$ in order. For two sets of identically indexed subsets of $\R^d$, $V_S = \{v_s\}_{s \in S}$ and $U_{S} = \{u_s\}_{s \in S}$, we abuse notation and let $\norm{V_S - U_S} = \sqrt{\sum_{s \in S} \norm{u_s - v_s}^2}$. For a set $\mc{W} \subset \R^d$, $\diam (\mc{W})$ will denote its diameter. A scalar function, $f$, when applied to a vector is applied elementwise. For $x \in \R^d$ and $S \subset[d]$, we let $x_S$ to denote the vector obtained by setting the entries of $x$ not in $S$ to $0$. For $\alpha \in (0, 1)$ and finite $S \subset \R$, we use $\quant_\alpha(S)$ to denote the $\alpha$th quantile of the set and $\phi$ and $\Phi$ will denote the pdf and cdf of a standard Gaussian random variable.
\section{Proof Overview} % \ynote{This needs to change \textcolor{red}{Should be okay now}}
\label{sec:proof_overview}

We now briefly describe the main ideas behind the proof of \cref{thm:lip_conc}. Before we begin, it is instructive to inspect standard methods of establishing similar results for the Gaussian setting and their shortcomings in our scenario. Specifically, for $n$ i.i.d standard Gaussian vectors $\{g_i\}_{i = 1}^n$ (the rows of a Gaussian matrix), our goal will be to establish the following inequality for some absolute constant $C > 0$:
\begin{equation}
    \label{eq:gau_conc}
    Z\lprp{\lbrb{g_i}_{i = 1}^n} \coloneqq \max_{\norm{v} \leq 1} \abs*{\frac{1}{n} \cdot \sum_{i = 1}^n f \lprp{\inp{g_i}{v}} - \E_{g \thicksim \mc{N} (0, \norm{v}^2)} \lsrs{f (g)}} \leq C \lprp{\sqrt{\frac{d}{n}} + \sqrt{\frac{\log 1 / \delta}{n}}}
\end{equation}
with probability at least $1 - \delta$. Note that $n$ here is analogous to $md$ for our RHTs.

We now bound the expectation and concentration terms of $Z$ separately. We start with the mildly more involved concentration term. In particular, we will show that $Z$ is a Lipschitz function of the $g_i$. Let $\{g^\prime_i\}_{i = 1}^n \in \R^d$ be an alternative choice of vectors. We have:
\begin{align*}
    &\abs*{Z(\{g_i\}_{i = 1}^n) - Z(\{g_i^\prime\}_{i = 1}^n)} \\
    &= \abs*{\max_{\norm{v} \leq 1} \abs*{\frac{1}{n} \cdot \sum_{i = 1}^n f \lprp{\inp{g_i}{v}} - \E_{g \thicksim \mc{N} (0, \norm{v}^2)} \lsrs{f (g)}} - \max_{\norm{v} \leq 1} \abs*{\frac{1}{n} \cdot \sum_{i = 1}^n f \lprp{\inp{g_i^\prime}{v}} - \E_{g \thicksim \mc{N} (0, \norm{v}^2)} \lsrs{f (g)}}} \\
    &\leq \max_{\norm{v} \leq 1} \abs*{\frac{1}{n} \cdot \sum_{i = 1}^n f \lprp{\inp{g_i}{v}} - \E_{g \thicksim \mc{N} (0, \norm{v}^2)} \lsrs{f (g)} - \frac{1}{n} \cdot \sum_{i = 1}^n f \lprp{\inp{g_i^\prime}{v}} + \E_{g \thicksim \mc{N} (0, \norm{v}^2)} \lsrs{f (g)}} \\
    &= \max_{\norm{v} \leq 1} \abs*{\frac{1}{n} \cdot \sum_{i = 1}^n (f \lprp{\inp{g_i}{v}} - f \lprp{\inp{g_i^\prime}{v}})} \leq \max_{\norm{v} \leq 1} \frac{1}{n} \cdot \sum_{i = 1}^n \abs*{f \lprp{\inp{g_i}{v}} - f \lprp{\inp{g_i^\prime}{v}}} \\
    &\leq \max_{\norm{v} \leq 1} \frac{1}{n} \cdot \sum_{i = 1}^n \abs*{\inp{g_i}{v}- \inp{g_i^\prime}{v}} \leq \frac{1}{n} \cdot \sum_{i = 1}^n \norm{g_i - g_i^\prime} \leq \frac{1}{\sqrt{n}} \cdot \norm{\{g_i\}_{i = 1}^n - \{g_i^\prime\}_{i = 1}^n}.
\end{align*}
The above display establishes that $Z$ is a $n^{-1/2}$-Lipschitz function of $\{g_i\}_{i = 1}^n$. Hence, we have:
\begin{equation*}
    \abs*{Z(\{g_i\}_{i = 1}^n) - \E [Z(\{g_i\}_{i = 1}^n)]} \leq \sqrt{\frac{2 \log 2 / \delta}{n}}
\end{equation*}
with probability at least $1 - \delta$ by concentration of Lipschitz functions of Gaussians (\cref{thm:tsirelson}). 

Letting $g^\prime_i \thicksim \mc{N}(0, I)$ and $\sigma_i \thicksim \{\pm 1\}$ be mutually independent standard normal vectors and Rademacher random variables respectively, we bound the expectation of $Z$ as follows:
\begin{align*}
    \E [Z] &= \E \lsrs{\max_{\norm{v} \leq 1} \abs*{\frac{1}{n} \cdot \sum_{i = 1}^n f \lprp{\inp{g_i}{v}} - \E_{g \thicksim \mc{N} (0, \norm{v}^2)} \lsrs{f (g)}}} \leq \E_{g_i, g_i^\prime} \lsrs{\max_{\norm{v} \leq 1} \abs*{\frac{1}{n} \cdot \sum_{i = 1}^n f \lprp{\inp{g_i}{v}} - f (\inp{g_i^\prime}{v})}} \\
    &= \E_{g_i, g_i^\prime, \sigma_i} \lsrs{\max_{\norm{v} \leq 1} \abs*{\frac{1}{n} \cdot \sum_{i = 1}^n \sigma_i \lprp{f \lprp{\inp{g_i}{v}} - f (\inp{g_i^\prime}{v})}}} \leq 2 \E_{g_i, \sigma_i} \lsrs{\max_{\norm{v} \leq 1} \abs*{\frac{1}{n} \cdot \sum_{i = 1}^n \sigma_i f \lprp{\inp{g_i}{v}}}} \\
    &\leq 4 \E_{g_i, \sigma_i} \lsrs{\max_{\norm{v} \leq 1} \abs*{\frac{1}{n} \cdot \sum_{i = 1}^n \sigma_i \inp{g_i}{v}}} = 4 \E_{g_i} \lsrs{\norm*{\frac{1}{n} \cdot \sum_{i = 1}^n g_i}} \leq 4 \sqrt{\frac{d}{n}}
\end{align*}
where the second to last inequality follows from Ledoux-Talagrand contraction (\cite[Theorem~4.12]{ledtal}). The previous two displays now yield \cref{eq:gau_conc}. This succinct argument, unfortunately, breaks down when working with RHTs in the place of Gaussians. While the concentration term can be modified to yield a weaker bound with $m$ in the denominator, the expectation term crucially relies on the mutual independence of all the $g_i$ and $g^\prime_i$ which does not hold true for RHTs.

An alternative approach is to use a standard gridding argument. Consider a $\gamma$-net, $\mc{G}$, of $\B(0, 1)$ for some small $\gamma$ (\cref{def:eps_net}). For each $v \in \mc{G}$, we have by noting that $f(\inp{v}{g_i})$ is $\norm{v}$-subGaussian:
\begin{equation*}
    \abs*{\frac{1}{n} \sum_{i = 1}^n f \lprp{\inp{v}{g_i}} - \E_{g \thicksim \mc{N} (0, \norm{v}^2)} [f \lprp{g}]} \leq \sqrt{\frac{2 \log 2 / \delta^\prime}{n}}
\end{equation*}
with probability at least $1 - \delta^\prime$. Setting $\delta^\prime = \delta / \abs{\mc{G}}$ and an application of the union bound yield the desired conclusion on the net as $\mc{G}$ can be chosen to satisfy $\abs{\mc{G}} \leq (C \gamma^{-1})^{d}$ (\cref{lem:cov_bnd}). Unfortunately, this simple argument also fails when working with RHTs. To see this, consider the case $v = e_1$. Here, the previous application of Hoeffding's Inequality yields the weaker inequality:
\begin{equation*}
    \abs*{\frac{1}{md} \sum_{i = 1}^m \sum_{j = 1}^d f \lprp{\wt{h}_{i, j} (v)} - \E_{g \thicksim \mc{N} (0, \norm{v}^2)} [f \lprp{g}]} \leq \sqrt{\frac{2 \log 2 / \delta^\prime}{m}}.
\end{equation*}
A naive union bound would then require $md = \Omega(d^2)$ (whereas our aim is to have $md$ nearly linear in $d$). This is reminiscent of the situation in compressed sensing, in which a naive union bound provides a similarly weak result \cite{CandesT05,RudelsonV08}. In the next subsection, we present our approach to circumvent this issue in our context (which is not related to the chaining technique used in the compressed sensing context).

\subsection{Our Approach}
\label{ssec:our_approach}

The first key observation underlying our approach is that while standard basis vectors lead to sub-optimal tail bounds, a typical vector behaves quite differently. For example, consider the vector $v = \bm{1} / \sqrt{d}$. In this case, it is not hard to show that each entry of $\wt{h} (v)$ is independent due to the orthogonality of the rows of $H_d$. Hence, for this particular vector, we obtain with probability at least $1 - \delta^\prime$:
\begin{equation*}
    \abs*{\frac{1}{md} \sum_{i = 1}^m \sum_{j = 1}^d f (\wt{h}_{i, j} (v)) - \E_{g \thicksim \mc{N} (0, \norm{v}^2)} \lsrs{f(g)}} \leq \sqrt{\frac{2 \log 2 / \delta^\prime}{md}}.
\end{equation*}
Since ``most'' vectors on the unit sphere are typically closer to $v$ than a standard basis vector, one could hope to apply the stronger inequality for most vectors while treating sparse vectors like those in the standard basis separately. Intuitively, our proof establishes the following concentration inequality:
\begin{equation}
    \label{eq:conc_inf_ov}
    \abs*{\frac{1}{md} \sum_{i = 1}^m \sum_{j = 1}^d f (\wt{h}_{i, j} (v)) - \E_{g \thicksim \mc{N} (0, \norm{v}^2)} \lsrs{f(g)}} \leq C \sqrt{\frac{\norm{v}_\infty^2 \log 1 / \delta^\prime}{m}}.
\end{equation}
Observe that the above inequality interpolates between the settings $v = e_1$ and $v = \bm{1} / \sqrt{d}$ depending on how well-spread the input vector is. We now use the inequality to establish our result for a simpler set of vectors.

Consider the following sets 
\begin{gather*}
    \forall S \subseteq [d]: \mc{G}_S := \lbrb{v \in \B(0, 1): \forall i,j \in S\, \frac{1}{2}\abs{v_i} \leq \abs{v_j} \leq 2 \abs{v_i}, \forall i \notin S\, v_i = 0} \\
    \forall k \in [d]: \mc{G}_k := \cup_{\substack{S \subset [d] \\ \abs*{S} = k}} \mc{G}_S  \\
    \mc{G} := \cup_{k \in [d]} \mc{G}_k
\end{gather*}
Hence, $\mc{G}_S$ consists of vectors uniformly spread over $S$ and for any $v \in \mc{G}_S$, we have $\norm{v}_\infty \leq 2 / \sqrt{\abs*{S}}$. We will use \cref{eq:conc_inf_ov} to perform a union bound over $\mc{G}$. First, notice that a $\gamma$-net of $\mc{G}_{S}$ has size $(C / \gamma)^{\abs{S}}$ and hence, there exists a $\gamma$-net of $\mc{G}_k$, $\wt{\mc{G}}_k$ of size at most $(Cd / \gamma)^{k}$. A union bound over only the elements in $\wt{\mc{G}}_k$ yields:
\begin{equation*}
    \forall v \in \wt{\mc{G}}_k: \abs*{\frac{1}{md} \sum_{i = 1}^m \sum_{j = 1}^d f (\wt{h}_{i, j} (v)) - \E_{g \thicksim \mc{N} (0, \norm{v}^2)} \lsrs{f(g)}} \leq C \sqrt{\frac{\norm{v}_\infty^2 \log 1 / \delta + \log (d / \gamma)}{m}}
\end{equation*}
with probability at least $1 - \delta / d$. Ignoring discretization errors, this establishes our concentration result for the restricted set $\mc{G}$. While $\mc{G}$ is quite restricted and this inequality is not strong enough to prove \cref{thm:lip_conc}, the ideas used in establishing it will play a key part in proving the general result.

Our next key observation is that any $v \in \B(0, 1)$ can be well approximated by a linear combination of a small number of vectors from $\mc{G}$; that is, $v \approx \sum_{i = 1}^r v_i$ for $r \approx \log (d / \eps)$ and $v_i \in \mc{G}$ with $\norm{v_i}_0 \leq \norm{v_{i + 1}}_0$. While the previously established inequalities are strong enough to ensure the conclusion of \cref{thm:lip_conc} for the individual components, $v_i$, this does not ensure that their combination enjoys similar concentration properties and it is not clear how these vectors behave when their embeddings are combined. 

The final ingredient in our argument is the stronger conditional inequality for $u = u_1 + u_2$ where $\supp (u_1) \cap \supp (u_2) = \phi$ and $U = \{(D^j)_{k,k}\}_{j \in [m], k \in \supp (u_2)}$:
\begin{equation}
    \label{eq:conc_inf_cond_ov}
    \abs*{\frac{1}{md} \sum_{i = 1}^m \sum_{j = 1}^d f(\wt{h}_{i, j} (u)) - \E \lsrs{\frac{1}{md} \sum_{i = 1}^m \sum_{j = 1}^d f(\wt{h}_{i, j} (u)) \mid U}} \leq C \sqrt{\frac{\norm{u_1}_\infty^2 \log 1 / \delta}{m}}.
\end{equation}
The above inequality shows that once we fix the variables in $U$, the concentration properties of the entries of $\wt{h} (u)$ are solely determined by the properties of $u_1$. This inequality allows us to establish uniformly over $v$:
\begin{equation}
    \label{eq:conc_cond_techniques}
    \forall k \in [r]: \abs*{\frac{1}{md} \sum_{i = 1}^m \sum_{j = 1}^d f (\wt{h}_{i, j} (v^k)) - \E \lsrs{\frac{1}{md} \sum_{i = 1}^m \sum_{j = 1}^d f(\wt{h}_{i, j} (v^k)) \mid V_{k - 1}}} \leq \frac{\eps}{r}
\end{equation}
where $v^k = \sum_{i = 1}^k v_i$ and $V_i = \{(D^j)_{k,k}\}_{j \in [m], k \in \supp (v^i)}$. The final step of our argument involves showing through a careful recursive argument that the conditional expectation is close to its unconditional expectation where we additionally require that \cref{eq:conc_cond_techniques} holds not just for the original function, $f$, but also for offset versions of the function, $f_t$, defined as $f_t(x) = f(x - t)$ for a large range of $t$. We prove \cref{eq:conc_inf_cond_ov} and carry out this argument in full detail in \cref{sec:lip_proof}.
\section{Proof of Uniform Lipschitz Concentration}
\label{sec:lip_proof}

In this section, we formally prove \cref{thm:lip_conc} by expanding on the outline presented in \cref{sec:proof_overview}. We begin by defining the class of functions for which our concentration properties will hold:
\begin{align*}
    \forall S \subseteq [d]: V_S &\coloneqq \bigcup_{j = 1}^m \lbrb{D^j_{i, i}}_{i \in S} \\
    \forall S \subset [d], t \in \R, z \in \R^d : F_{S, t} (z) &\coloneqq \frac{1}{md} \cdot \sum_{j = 1}^m \sum_{k = 1}^d f \lprp{\wt{h}_{j,k}(z_{S}) - t} \\
    \forall S, T \subseteq [d], t \in \R, z \in \R^d: \wt{F}_{S, T, t} (z) &\coloneqq \mb{E}\lsrs{F_{S \cup T, t} (z) \mid V_T} \tag{LIP-NOT} \label{eq:lip_not}
\end{align*}
To help define the nets used in our argument, we introduce the following notation and define:
\begin{gather*}
    \rho \coloneqq \lipRes, \ \lambda \coloneqq \lipWid,\ \gamma \coloneqq \lipWidLoss, \ \nu \coloneqq \lipAcc\\
    \forall S, T \subseteq [d]: \mc{G}_{S, T} \coloneqq \lbrb{z \in \R^d:  \norm{z} \leq 1,\ \forall i \notin S \cup T, z_i = 0,\ \forall i,j \in S, \frac{1}{2} \abs{z_j} \leq \abs{z_i} \leq 2 \abs{z_j}} \\
    r \coloneqq \lipNSplits,\quad \zeta \coloneqq \lipSens,\quad \mc{T} \coloneqq \lbrb{\pm i \cdot \zeta}_{i = 0}^{\lambda / \zeta} \tag{LIP-NET} \label{eq:lip_net}
\end{gather*}
For all disjoint $S, T \subseteq [d]$ such that $\abs{T} \leq r \cdot \abs{S}$, let $\wt{\mc{G}}_{S, T}$ be an $\rho$-net of $\mc{G}_{S,T}$ (\cref{def:eps_net}). Note, we may assume that $\abs{\wt{\mc{G}}_{S,T}} \leq \lprp{\frac{10}{\rho}}^{(r + 1) \cdot \abs{S}}$ (\cref{lem:cov_bnd}). We now have the following claim:
\begin{claim}
    \label{clm:net_part_conc_lip}
    We have:
    \begin{equation*}
        \forall t \in \mc{T}, S,T \text{ s.t } \abs{T} \leq r \cdot \abs{S}, T \cap S = \phi, z \in \wt{\mc{G}}_{S,T}: \abs{F_{S \cup T, t} (z) - \wt{F}_{S, T, t} (z)} \leq \nu
    \end{equation*}
    with probability at least $1 - \delta / 4$.
\end{claim}
\begin{proof}
    For the proof, we start by conditioning on $V_T$. Note that for all $i \in S$, we must have:
    \begin{equation*}
        \abs{z_i} \leq \frac{2}{\sqrt{\abs{S}}}.
    \end{equation*}
    Let $V^1_S$ and $V^2_S$ be two distinct settings of the random variables $V_S$ and let $F^1_{S, t} (x), F^2_{S, t} (x)$ be the values of $F_{S, t}$ computed by setting the variables in $V_S$ to $V^1_S$ and $V^2_S$ respectively fixing all the rest to be the same. Similarly, let $\wt{h}^1(\cdot), \wt{h}^2(\cdot)$ denote the vectors $\wt{h}$ computed with the corresponding settings of $V_S$ and $D^{j,1}, D^{j,2}$ be the corresponding diagonal matrices for $j \in [m]$. Recalling that $f(\cdot)$ is a $1$-Lipschitz function, we have:
    \begin{align*}
        &\abs{F^1_{S \cup T, t} (z) - F^2_{S \cup T, t} (z)} \leq \frac{1}{md} \cdot \sum_{j = 1}^m \sum_{k = 1}^d \abs{\wt{h}^1_{j,k} (z) - \wt{h}^2_{j,k} (z)} \leq \frac{1}{\sqrt{md}} \cdot \sqrt{\sum_{j = 1}^m \sum_{k = 1}^d (\wt{h}^1_{j,k} (z) - \wt{h}^2_{j,k} (z))^2} \\
        &= \frac{1}{\sqrt{md}} \cdot 
        \norm*{
            \begin{bmatrix}
                H (D^{1, 1} - D^{1, 2}) z \\
                H (D^{2, 1} - D^{2, 2}) z \\
                \vdots \\
                H (D^{m, 1} - D^{m, 2}) z
            \end{bmatrix}
        } = \frac{1}{\sqrt{md}} \cdot 
        \norm*{
            \begin{bmatrix}
                H & 0 & \dots & 0\\
                0 & H & \dots & 0\\
                \vdots & \vdots & \ddots & \vdots\\
                0 & 0 & \dots & H
            \end{bmatrix} \cdot 
            \begin{bmatrix}
                (D^{1,1} - D^{1,2}) z \\
                (D^{2,1} - D^{2,2}) z \\
                \vdots \\
                (D^{m,1} - D^{m,2}) z
            \end{bmatrix}} \\
        &= \frac{1}{\sqrt{m}} \cdot 
        \norm*{
            \begin{bmatrix}
                (D^{1,1} - D^{1,2}) z \\
                (D^{2,1} - D^{2,2}) z \\
                \vdots \\
                (D^{m,1} - D^{m,2}) z
            \end{bmatrix}
        } = 
        \frac{1}{\sqrt{m}} \cdot 
        \norm*{
            \begin{bmatrix}
                \Diag (z) \diag(D^{1,1} - D^{1,2}) \\
                \Diag (z) \diag(D^{2,1} - D^{2,2}) \\
                \vdots \\
                \Diag (z) \diag(D^{m,1} - D^{m,2})
            \end{bmatrix}
        } \leq 2 \cdot \frac{1}{\sqrt{m}} \cdot \frac{1}{\sqrt{\abs{S}}} \cdot \norm{V^1_S - V^2_S}
    \end{align*}
    Therefore, $F_{S \cup T, t} (z)$ is a $\frac{2}{\sqrt{m \abs{S}}}$-Lipschitz function of $V_S$ conditioned on $V_T$. Hence, we have by \cref{thm:tsirelson}:
    \begin{equation*}
        \P \lbrb{\abs{F_{S \cup T, t} (z) - \wt{F}_{S, T, t} (z)} \leq \nu} \geq 1 - \frac{\delta}{32 \cdot (10 / \rho)^{(r + 1)\cdot \abs{S}}\cdot (d + 1)^{(r + 1) \cdot \abs{S}} \cdot \abs{\mc{T}} \cdot d^2}.
    \end{equation*}
    The claim now follows from a union bound over all possible $S, T$ satisfying the constraints.
\end{proof}

From this point on, we condition on the conclusions of \cref{lem:h_spec_bnd,clm:net_part_conc_lip}; i.e we condition on the following event which occurs with probability at least $1 - \delta / 2$ via \cref{clm:net_part_conc_lip,lem:h_spec_bnd}: 
\begin{gather*}
    \forall t \in \mc{T}, S,T \text{ s.t } \abs{T} \leq r \cdot \abs{S}, T \cap S = \phi, z \in \wt{\mc{G}}_{S,T}: \abs{F_{S \cup T, t} (z) - \wt{F}_{S, T, t} (z)} \leq \nu \\
    \forall x, y \in \R^d: \frac{\norm*{\wt{h} (x) - \wt{h} (y)}}{\sqrt{md}} \leq 2 \norm{x - y}
\end{gather*}

We now extend the conclusion of \cref{clm:net_part_conc_lip} to all $z \in \mc{G}_{S,T}$.
\begin{claim}
    \label{clm:part_conc_lip}
    We have:
    \begin{equation*}
        \forall t \in \mc{T}, S,T \text{ s.t } \abs{T} \leq r \cdot \abs{S}, z \in \mc{G}_{S,T},\ T \cap S = \phi : \abs{F_{S \cup T, t} (z) - \wt{F}_{S, T, t} (z)} \leq 2\nu.
    \end{equation*}
\end{claim}
\begin{proof}
    Let $z \in \mc{G}_{S,T}, t \in \mc{T}$ for $S,T$ satisfying the constraints and $\wt{z} = \argmin_{u \in  \wt{\mc{G}}_{S,T}} \norm{z - u}$. Note that $\norm{z - \wt{z}} \leq \rho$. We simply show that $F_{S \cup T, t}, \wt{F}_{S, T, t}$ are close to their corresponding values for $\wt{z}$. For the first term (i.e for $F_{S \cup T, t}$), we have:
    \begin{align*}
        \abs{F_{S \cup T, t}(z) - F_{S \cup T, t}(\wt{z})} &= \frac{1}{md} \cdot \abs*{\sum_{j = 1}^m \sum_{k = 1}^d f(\wt{h}_{j, k} (z_{S \cup T}) - t) - f(\wt{h}_{j, k} (\wt{z}_{S \cup T}) - t)} \\
        &\leq \frac{1}{md} \cdot \sum_{j = 1}^m \sum_{k = 1}^d \abs*{f(\wt{h}_{j, k} (z_{S \cup T}) - t) - f(\wt{h}_{j, k} (\wt{z}_{S \cup T}) - t)} \\
        &\leq \frac{1}{md} \cdot \sum_{j = 1}^m \sum_{k = 1}^d \abs*{\wt{h}_{j,k}(z_{S \cup T}) - \wt{h}_{j,k}(\wt{z}_{S \cup T})} \\
        &\leq \frac{1}{md} \cdot \sqrt{md} \cdot \sqrt{\sum_{j = 1}^m \sum_{k = 1}^d \lprp{\wt{h}_{j, k}(z_{S \cup T}) - \wt{h}_{j, k}(\wt{z}_{S \cup T})}^2} \\
        &\leq \frac{1}{\sqrt{md}} \cdot 2\sqrt{md} \cdot \rho \leq \frac{\nu}{4}
    \end{align*}
    concluding the proof for the first term. For the second term (i.e $\wt{F}_{S, T, t}$), we proceed as follows:
    \begin{align*}
        \abs*{\wt{F}_{S, T, t} (z) - \wt{F}_{S, T, t} (\wt{z})} &= \abs*{\E \lsrs{F_{S \cup T, t}(z) - F_{S \cup T, t} (\wt{z})\mid V_T}} \\
        &\leq \frac{1}{md} \cdot \sum_{j = 1}^m \sum_{k = 1}^d \mb{E} \lsrs{\abs{f(\wt{h}_{j, k} (z_{S \cup T}) - t) - f(\wt{h}_{j, k} (\wt{z}_{S \cup T}) - t)} \mid V_T} \\
        &\leq \frac{1}{md} \cdot \sum_{j = 1}^m \sum_{k = 1}^d \E \lsrs{\abs{\wt{h}_{j,k} (z_{S \cup T}) - \wt{h}_{j,k} (\wt{z}_{S \cup T})} \mid V_T} \\
        &\leq \frac{1}{md} \cdot \sum_{j = 1}^m \sum_{k = 1}^d \E \lsrs{\abs{\wt{h}_{j,k} (z_{S}) - \wt{h}_{j,k} (\wt{z}_{S})} + \abs{\wt{h}_{j,k} (z_{T}) - \wt{h}_{j,k} (\wt{z}_{T})}\mid V_T} \\
        &\leq \frac{1}{md} \cdot \lprp{\sum_{j = 1}^m \sum_{k = 1}^d \rho \cdot \E_{g \thicksim \mc{N}(0, 1)} \lsrs{\abs{g}} + \abs{\wt{h}_{j, k} (z_T) - \wt{h}_{j, k} (\wt{z}_T)}} \\
        &\leq \rho \E_{g \thicksim \mc{N}(0, 1)} \lsrs{\abs{g}} + \frac{1}{md} \cdot \sqrt{md} \cdot \sqrt{\sum_{j = 1}^m \sum_{k = 1}^d (\wt{h}_{j, k} (z_T) - \wt{h}_{j, k}(\wt{z}_T))^2} \\
        &\leq \frac{\nu}{4} + \frac{1}{\sqrt{md}} \cdot 2\sqrt{md} \cdot \rho \leq \frac{\nu}{2}.
    \end{align*}
    The previous two bounds along with the conclusion of \cref{clm:net_part_conc_lip} yield the claim.
\end{proof}

Let $z \in \R^d$ with $\norm{z} \leq 1$. We decompose the coordinates of $z$ into disjoint sets defined as follows:
\begin{equation*}
    \tau^* \coloneqq \max_{k \in [d]} \abs{z_k}, \quad \forall l \in [r]: S_l \coloneqq \lbrb{k: \frac{1}{2^l} \cdot \tau^* < \abs{z_k} \leq \frac{1}{2^{l - 1}} \cdot \tau^*}.
\end{equation*}
Let $S_{(1)}, \dots, S_{(r)}$ be an ordering of the $S_l$ such that $\abs{S_{(i)}} \leq \abs{S_{(i + 1)}}$ for all $i \in [r - 1]$ and define the sets $S^l$:
\begin{equation*}
    \forall l \in [r]: S^l \coloneqq \cup_{q = 1}^l S_{(q)} \text{ and } T \coloneqq [d] \setminus S^r.
\end{equation*}

We now show that as far as the functions $F$ are concerned, $z$ is well approximated by $z_{S^r}$.
\begin{claim}
    \label{clm:z_zsr_lip}
    We have for all $t \in \R$:
    \begin{gather*}
        \abs*{\E \lsrs{F_{[d], t} (z) - F_{S^r, t} (z_{S^r})}} \leq \frac{\nu}{4}\\
        \abs*{F_{[d], t} (z) - F_{S^r, t} (z_{S^r})} \leq \frac{\nu}{4}.
    \end{gather*}
\end{claim}
\begin{proof}
    We make the following simple observation:
    \begin{equation}
        \label{eq:part_rem_bnd_lip}
        \norm{z_{T}} \leq \sqrt{d} \cdot \frac{1}{2^r} \leq \rho.
    \end{equation}
    From \cref{eq:part_rem_bnd_lip}, we have:
    \begin{align*}
        \abs*{\E \lsrs{F_{[d], t} (z) - F_{S^r, t} (z_{S^r})}} &= \abs*{\E \lsrs{F_{[d], t} (z) - F_{[d], t} (z_{S^r})}} \\
        &\leq \frac{1}{md} \cdot \sum_{j = 1}^m \sum_{k = 1}^d \E \lsrs{\abs*{f(\wt{h}_{j,k} (z) - t) - f(\wt{h}_{j,k} (z_{S^r}) - t)}} \\
        &\leq \frac{1}{md} \cdot \sum_{j = 1}^m \sum_{k = 1}^d \E \lsrs{\abs*{\wt{h}_{j,k} (z) - \wt{h}_{j,k} (z_{S^r})}} \leq \rho \cdot \E_{g \thicksim \mc{N}(0, 1)} [\abs{g}] \leq \frac{\nu}{4}.
    \end{align*}
    Similarly, we have:
    \begin{align*}
        \abs*{F_{[d], t} (z) - F_{S^r, t} (z_{S^r})} &= \abs*{F_{[d], t} (z) - F_{[d], t} (z_{S^r})} \\
        &\leq \frac{1}{md} \cdot \sum_{j = 1}^m \sum_{k = 1}^d \abs*{f(\wt{h}_{j,k} (z) - t) - f(\wt{h}_{j,k} (z_{S^r}) - t)} \\
        &\leq \frac{1}{md} \cdot \sum_{j = 1}^m \sum_{k = 1}^d \abs*{\wt{h}_{j,k} (z) - \wt{h}_{j,k} (z_{S^r})} \leq \frac{1}{md} \cdot \sqrt{md} \cdot \sqrt{\sum_{j = 1}^m \sum_{k = 1}^d (\wt{h}_{j,k} (z) - \wt{h}_{j,k} (z_{S^r}))^2} \\
        &\leq \frac{1}{\sqrt{md}} \cdot 2\sqrt{md} \cdot \rho \leq \frac{\nu}{4}
    \end{align*}
    concluding the proof of the claim.    
\end{proof}

In our final claim, we show that $F_{[d], t} (z_{S^r})$ is close to its expectation. % \ynote{Fix this lemma. It isn't true for all $t \in \R$. \textcolor{red}{Should be correct now.}}
\begin{claim}
    \label{clm:approx_exp_close_lip}
    We have:
    \begin{equation*}
        \forall l \in [r], t \in \R \text{ s.t } \abs{t} \leq \lambda - (l - 1)\gamma : \abs*{F_{S^l, t} (z_{S^l}) - \E \lsrs{F_{S^l, t} (z_{S^l})}} \leq 3l\nu.
    \end{equation*}
    and consequently, from \cref{clm:z_zsr_lip}:
    \begin{equation*}
        \forall t \in \R \text{ s.t } \abs{t} \leq \lambda -r \gamma: \abs*{F_{[d], t} (z) - \E \lsrs{F_{[d], t} (z)}} \leq 3(r + 1)\nu.
    \end{equation*}
\end{claim}
\begin{proof}
    We prove the claim inductively and start with the base case of the induction.
    \paragraph{Base case:} \cref{clm:part_conc_lip} establishes the base case for all $t \in \mc{T}$. Now, for any $\abs{t} \leq \lambda$, there exists $t^\prime \in \mc{T}$ with $\abs{t^\prime - t} \leq \zeta$. For this $t^\prime$, we have by the triangle inequality and the fact that $F$ is $1$-Lipschitz in $t$:
    \begin{align*}
        &\abs*{F_{S^1, t} (z_{S^1}) - \E \lsrs{F_{S^1, t} (z_{S^1})}} \\
        &\leq \abs*{F_{S^1, t^\prime} (z_{S^1}) - \E \lsrs{F_{S^1, t^\prime} (z_{S^1})}} + \abs*{F_{S^1, t} (z_{S^1}) - F_{S^1, t^\prime} (z_{S^1}) - \E \lsrs{F_{S^1, t} (z_{S^1}) - F_{S^1, t^\prime} (z_{S^1})}} \\
        &\leq \abs*{F_{S^1, t^\prime} (z_{S^1}) - \E \lsrs{F_{S^1, t^\prime} (z_{S^1})}} + \abs*{F_{S^1, t} (z_{S^1}) - F_{S^1, t^\prime} (z_{S^1})} + \abs*{\E \lsrs{F_{S^1, t} (z_{S^1}) - F_{S^1, t^\prime} (z_{S^1})}} \\
        &\leq \abs*{F_{S^1, t^\prime} (z_{S^1}) - \E \lsrs{F_{S^1, t^\prime} (z_{S^1})}} + \abs*{F_{S^1, t} (z_{S^1}) - F_{S^1, t^\prime} (z_{S^1})} + \E \lsrs{\abs*{F_{S^1, t} (z_{S^1}) - F_{S^1, t^\prime} (z_{S^1})}} \\
        &\leq 2\nu + 2 \zeta \leq 3\nu
    \end{align*}
    establishing the base case of the induction.

    \paragraph{Inductive case:} For the induction step, we proceed similarly to the base case with one additional step. Assuming the induction up to $l = q$, we establish it for $l = q + 1$. First, we show that $\wt{F}_{S_{(q + 1)}, S^q, t} (z_{S^{q + 1}})$ is close to $\E \lsrs{F_{S^{q + 1}, t} (z_{S^{q + 1}})}$ for all $\abs{t} \leq \lambda - q \gamma$. We proceed as follows:
    \begin{align*}
        &\abs*{\wt{F}_{S_{(q + 1)}, S^q, t} (z_{S^{q + 1}}) - \E \lsrs{F_{S^{q + 1}, t}(z_{S^{q + 1}})}} \\
        &= \abs*{\frac{1}{md} \cdot \sum_{j = 1}^m \sum_{k = 1}^d \E \lsrs{f (\wt{h}_{j, k}(z_{S^{q + 1}}) - t) \mid V_{S^q}} - \E \lsrs{F_{S^{q + 1}, t}(z_{S^{q + 1}})}} \\
        &= \abs*{\frac{1}{md} \cdot \sum_{j = 1}^m \sum_{k = 1}^d \E_{g \thicksim \mc{N}(0, 1)} \lsrs{f (\norm{z_{S_{(q + 1)}}} \cdot g + \wt{h}_{j,k} (z_{S^{q}}) - t) \mid V_{S^q}} - \E \lsrs{F_{S^{q + 1}, t}(z_{S^{q + 1}})}} \\
        &= \abs*{\frac{1}{md} \cdot \sum_{j = 1}^m \sum_{k = 1}^d \int_{-\infty}^\infty f(\norm{z_{S_{(q + 1)}}} \cdot y + \wt{h}_{j,k} (z_{S^q}) - t) \phi(y) dy - \E \lsrs{F_{S^{q + 1}, t}(z_{S^{q + 1}})}} \\
        &= \abs*{\int_{-\infty}^\infty F_{S^q, t - \norm*{z_{S_{(q + 1)}}} \cdot y} (z_{S^q}) \phi(y) dy - \E_{y_1, y_2 \thicksim \mc{N}(0, 1)} \lsrs{f(\norm{z_{S_{(q + 1)}}}\cdot y_1 + \norm{z_{S^q}} \cdot y_2 - t)}} \\
        &= \abs*{\int_{-\infty}^\infty F_{S^q, t - \norm*{z_{S_{(q + 1)}}} \cdot y} (z_{S^q}) \phi(y) dy - \int_{-\infty}^\infty \int_{-\infty}^\infty f (\norm{z_{S^q}} \cdot y_1 + \norm{z_{S_{(q + 1)}}} \cdot y_2 - t) \phi(y_1) \phi(y_2)dy_1 dy_2} \\
        &= \abs*{\int_{-\infty}^\infty \lprp{F_{S^q, t - \norm*{z_{S_{(q + 1)}}} \cdot y} (z_{S^q}) - \E \lsrs{F_{S^q, t - \norm*{z_{S_{(q + 1)}}} \cdot y} (z_{S^q})}} \phi(y) dy} \\
        &\leq \int_{-\infty}^\infty \abs*{F_{S^q, t - \norm*{z_{S_{(q + 1)}}} \cdot y} (z_{S^q}) - \E \lsrs{F_{S^q, t - \norm*{z_{S_{(q + 1)}}} \cdot y} (z_{S^q})}} \phi(y) dy \\
        &= \underbrace{\int_{-\gamma}^\gamma \abs*{F_{S^q, t - \norm*{z_{S_{(q + 1)}}} \cdot y} (z_{S^q}) - \E \lsrs{F_{S^q, t - \norm*{z_{S_{(q + 1)}}}\cdot y} (z_{S^q})}} \phi(y) dy}_{\alpha} + \\
        &\qquad \underbrace{\int_{(-\infty, -\gamma] \cup [\gamma, \infty)} \abs*{F_{S^q, t - \norm*{z_{S_{(q + 1)}}} \cdot y} (z_{S^q}) - \E \lsrs{F_{S^q, t - \norm*{z_{S_{(q + 1)}}}\cdot y} (z_{S^q})}} \phi(y) dy}_{\beta} \addtocounter{equation}{1} \tag{\theequation} \label{eq:ind_decomp}\\
    \end{align*}
    For the first term, $\alpha$, we have by the inductive hypothesis as $\abs{t - \norm{z_{S_{(q + 1)}}} y} \leq \lambda - (q - 1) \gamma$ when $\abs{y} \leq \gamma$:
    \begin{equation*}
        \alpha \coloneqq \int_{-\gamma}^\gamma \abs*{F_{S^q, t - \norm*{z_{S_{(q + 1)}}} \cdot y} (z_{S^q}) - \E \lsrs{F_{S^q, t - \norm*{z_{S_{(q + 1)}}}\cdot y} (z_{S^q})}} \phi(y) dy \leq 3q\nu.
    \end{equation*}
    For the second term, $\beta$, we proceed as follows noting $F_{S, t} (z)$ is $1$-Lipschitz in $t$ for all $z \in \R^d$:
    \begin{align*}
        \beta &\leq \int_{(-\infty, -\gamma] \cup [\gamma, \infty)} \abs*{F_{S^q, t - \norm*{z_{S_{(q + 1)}}} \cdot y} (z_{S^q}) - \E \lsrs{F_{S^q, t - \norm*{z_{S_{(q + 1)}}}\cdot y} (z_{S^q})}} \phi(y) dy \\
        &\leq \int_{(-\infty, -\gamma] \cup [\gamma, \infty)} \abs*{F_{S^q, 0} (z_{S^q}) - \E \lsrs{F_{S^q, 0} (z_{S^q})}} \phi(y) dy + \\
        &\qquad  \int_{(-\infty, -\gamma] \cup [\gamma, \infty)} \abs*{F_{S^q, t - \norm*{z_{S_{(q + 1)}}} \cdot y} (z_{S^q}) - F_{S^q, 0} (z_{S^q})  - \E \lsrs{F_{S^q, t - \norm*{z_{S_{(q + 1)}}} \cdot y} (z_{S^q}) - F_{S^q, 0} (z_{S^q})}} \phi(y) dy \\
        &\leq \int_{(-\infty, -\gamma] \cup [\gamma, \infty)} \lprp{3q\nu + 2\abs*{t - \norm{z_{S_{(q + 1)}}} \cdot y}} \phi (y) dy \leq 3 \int_{(-\infty, -\gamma] \cup [\gamma, \infty)} \lprp{q\nu + \abs{t} + \norm{z_{S_{(q + 1)}}} \cdot \abs{y}} \phi (y) dy \\
        &\leq 6 \int_{(-\infty, -\gamma] \cup [\gamma, \infty)} \lprp{\lambda + \abs{y}} \phi (y) dy \leq \frac{\nu}{2}
    \end{align*}
    where the last inequality follows the setting of $\gamma, \lambda, \nu$ (\ref{eq:lip_net}). Putting the previous two bounds together:
    \begin{equation*}
        \abs*{\wt{F}_{S_{(q + 1)}, S^q, t} (z_{S^{q + 1}}) - \E \lsrs{F_{S^{q + 1}, t}(z_{S^{q + 1}})}} \leq 3q\nu + \frac{\nu}{2}.
    \end{equation*}
    Now, as in the base case, we simply bound the deviations of $F$ from its expectation. \cref{clm:part_conc_lip} now yields:
    \begin{equation*}
        \forall t \in \mc{T}: \abs*{F_{S^{q + 1}, t} (z_{S^{q + 1}}) - \E \lsrs{F_{S^{q + 1}, t} (z_{S^{q + 1}})}} \leq \frac{5\nu}{2} + 3q\nu.
    \end{equation*}

    Similarly to the base case, the previous display establishes the inductive hypothesis for all $t \in \mc{T}$. For any $t$ such that $\abs{t} \leq \lambda - q\gamma$, there exists  $t^\prime \in \mc{T}$ with $\abs{t - t^\prime} \leq \zeta$. Then, we have:
    \begin{align*}
        &\abs*{F_{S^{q + 1}, t} (z_{S^{q + 1}}) - \E \lsrs{F_{S^{q + 1}, t} (z_{S^{q + 1}})}} \\
        &\leq \abs*{F_{S^{q + 1}, t^\prime} (z_{S^{q + 1}}) - \E \lsrs{F_{S^{q + 1}, t^\prime} (z_{S^{q + 1}})}} + \\
        &\qquad \abs*{F_{S^{q + 1}, t} (z_{S^{q + 1}}) - F_{S^{q + 1}, t^\prime} (z_{S^{q + 1}}) - \E \lsrs{F_{S^{q + 1}, t} (z_{S^{q + 1}}) - F_{S^{q + 1}, t^\prime} (z_{S^{q + 1}})}} \\
        &\leq 3q\nu + \frac{5\nu}{2} + \abs*{F_{S^{q + 1}, t} (z_{S^{q + 1}}) - F_{S^{q + 1}, t^\prime} (z_{S^{q + 1}})} + \abs*{\E \lsrs{F_{S^{q + 1}, t} (z_{S^{q + 1}}) - F_{S^{q + 1}, t^\prime} (z_{S^{q + 1}})}} \\
        &\leq 3q\nu + \frac{5\nu}{2} + \abs*{F_{S^{q + 1}, t} (z_{S^{q + 1}}) - F_{S^{q + 1}, t^\prime} (z_{S^{q + 1}})} + \E \lsrs{\abs*{F_{S^{q + 1}, t} (z_{S^{q + 1}}) - F_{S^{q + 1}, t^\prime} (z_{S^{q + 1}})}} \\
        &\leq 3q\nu + \frac{5\nu}{2} + 2 \zeta \leq 3(q + 1)\nu
    \end{align*}
    establishing the hypothesis for all $\abs{t} \leq \lambda - q\gamma$. The final statement of the claim now follows by an application of \cref{clm:z_zsr_lip} along with the above inductive hypothesis.
\end{proof}

\cref{thm:lip_conc} now follows from \cref{clm:approx_exp_close_lip} along with a union bound over \cref{clm:net_part_conc_lip,lem:h_spec_bnd}.
% From \cref{clm:approx_exp_close}, we may conclude the proof of the theorem with the following manipulations:
% \begin{align*}
%     &\frac{1}{md} \cdot \sum_{j = 1}^m \sum_{k = 1}^d \bm{1} \lbrb{\wt{h}_{j,k} (z) \leq t} - \P_{Z \sim \mc{N}(0, 1)} (Z \leq t) \\
%     &\leq F_{[d], t} (z) - \E \lsrs{F_{[d], t - \frac{1}{\eta}} (z)} \leq 3(r + 1)\nu + \E \lsrs{F_{[d], t} (z) - F_{[d], t - \frac{1}{\eta}} (z)} \\
%     &\leq 3(r + 1)\nu + \P_{Y \thicksim \mc{N}(0, 1)} \lbrb{Y \leq t + \frac{1}{\eta}} - \P_{Y \thicksim \mc{N} (0, 1)} \lbrb{Y \leq t - \frac{1}{\eta}} \leq 3(r + 1)\nu + \frac{2}{\eta} \leq 3(r + 2) \nu.
% \end{align*}
% Similarly, we have for the other direction:
% \begin{align*}
%     &\frac{1}{md} \cdot \sum_{j = 1}^m \sum_{k = 1}^d \bm{1} \lbrb{\wt{h}_{j,k} (z) \leq t} - \P_{Z \sim \mc{N}(0, 1)} (Z \leq t) \\
%     &\geq F_{[d], t - \frac{1}{\eta}} (z) - \E \lsrs{F_{[d], t} (z)} \geq - 3(r + 1)\nu - \E \lsrs{F_{[d], t} (z) - F_{[d], t - \frac{1}{\eta}} (z)} \\
%     &\geq -3(r + 1)\nu - \P_{Y \thicksim \mc{N}(0, 1)} \lbrb{Y \leq t + \frac{1}{\eta}} + \P_{Y \thicksim \mc{N} (0, 1)} \lbrb{Y \leq t - \frac{1}{\eta}} \geq -3(r + 1)\nu - \frac{2}{\eta} \geq -3(r + 2) \nu.
% \end{align*}
% concluding the proof of the theorem.
\qed
\section{Kernel Approximation Proof}
\label{sec:kern_approx_proof}
In this section, we prove \cref{thm:kern_approx} leveraging \cref{thm:lip_conc}. We start with a simple algebraic manipulation. For all $x, y \in \mb{R}^d$, we have by standard trigonometric identities:
\begin{align*}
    \inp{h(x)}{h(y)} &= \frac{2}{md} \cdot \sum_{j = 1}^m \sum_{k = 1}^d \cos \lprp{(H D^j x)_k + b^j_k} \cos \lprp{(H D^j y)_k + b^j_k} \\
    &= \frac{1}{md} \cdot \sum_{j = 1}^m \sum_{k = 1}^d \lprp{\cos \lprp{\lprp{HD^j (x + y)}_k + 2 b^j_k} + \cos \lprp{\lprp{HD^j (x - y)}_k}}. \tag{KER-DEC} \label{eq:kern_cos_decomp}
\end{align*}
We will show that the first term is uniformly close to $0$ for all $x, y \in \mc{W}$. This fact follows straightforwardly by using the fact that the $b^j_k$s are independent of the $D^j$s. Consequently, our efforts will primarily be focussed on the second term. The following simple lemma shows that the first term in \cref{eq:kern_cos_decomp} is uniformly close to $0$ for all $x, y \in \mc{W}$. 

\begin{lemma}
    \label{lem:cos_dec_first}
    For $m \geq \kernEmbedRep$, we have that:
    \begin{equation*}
        \forall x, y \in \mc{W}: \frac{1}{md} \cdot \abs*{\sum_{j = 1}^m \sum_{k = 1}^d \cos \lprp{\wt{h}_{j, k} (x + y) + 2b^j_k}} \leq \frac{\eps}{8}
    \end{equation*}
    with probability at least $1 - \delta / 2$.
\end{lemma}
\begin{proof}
    We have from \cref{lem:h_spec_bnd} that with probability at least $1 - \delta / 8$:
    \begin{equation*}
        \forall x, y \in \R^d: \norm*{\wt{h} (x) - \wt{h} (y)} \leq 2\sqrt{md}\cdot \norm*{x - y}.
    \end{equation*}
    Let $\mc{G}$ be a $\rho$-net of $\mc{H} = \{x + y: x,y \in \mc{W}\}$, with $\rho = \lprp{\frac{\eps}{10 \cdot d}}^{10}$. Note we may assume $\abs{\mc{G}} \leq \lprp{\frac{20 \cdot \diam (\mc{W})}{\rho}}^d$. For any $z \in \mc{G}$, we have from the independence of the $b^j_k$ from the $D^j$ and Hoeffding's Inequality:
    \begin{equation*}
        \P \lbrb{\abs*{\frac{1}{md} \cdot \sum_{j = 1}^m \sum_{k = 1}^d \cos \lprp{\wt{h}_{j,k} (z) + 2b^j_k}} \geq t} \leq 2 \exp \lbrb{- \frac{md t^2}{2}}.
    \end{equation*}
    Setting $t = \eps / 16$ and a  union bound over all $z \in \mc{G}$ yields that with probability at least $1 - \delta / 8$:
    \begin{equation*}
        \forall z \in \mc{G}: \abs*{\frac{1}{md} \cdot \sum_{j = 1}^m \sum_{k = 1}^d \cos \lprp{ \wt{h}_{j,k} (z) + 2b^j_k}} \leq \frac{\eps}{16}.
    \end{equation*}
    Now let $z \in \mc{H}$ with $\wt{z} = \argmin_{w \in \mc{G}} \norm*{z - w}$. Now, we have from the fact that $\cos (\cdot)$ is $1$-Lipschitz:
    \begin{align*}
        &\abs*{\frac{1}{md} \cdot \sum_{j = 1}^m \sum_{k = 1}^d \cos \lprp{ \wt{h}_{j,k} (z) + 2b_k^j} - \cos \lprp{ \wt{h}_{j,k} (\wt{z}) + 2b_k^j}} \\
        &\leq \frac{1}{md} \cdot \sum_{j = 1}^m \sum_{k = 1}^d \abs*{\cos \lprp{ \wt{h}_{j,k} (z) + 2b_k^j} - \cos \lprp{ \wt{h}_{j,k} (\wt{z}) + 2b_k^j}} \\
        &\leq \frac{1}{md} \cdot \sum_{j = 1}^m \sum_{k = 1}^d \abs*{\wt{h}_{j,k} (z) - \wt{h}_{j,k} (\wt{z})} \leq \frac{1}{\sqrt{md}} \cdot \lprp{\sum_{j = 1}^m \sum_{k = 1}^d \lprp{\wt{h}_{j,k} (z) - \wt{h}_{j,k} (\wt{z})}^2}^{1/2} \\
        &\leq \frac{2}{\sqrt{md}} \cdot 2\sqrt{md} \cdot \norm*{z - \wt{z}} \leq \frac{\eps}{16}.
    \end{align*}
    The previous two displays yield the conclusion of the lemma by a triangle inequality.
\end{proof}

We now prove a lemma which shows that the second term in \cref{eq:kern_cos_decomp} is close to its expectation. 

\begin{lemma}
    \label{lem:cos_dec_second}
    For $m \geq \kernEmbedRep$, we have:
    \begin{equation*}
        \forall x,y \in \mc{W} :  \abs*{\frac{1}{md} \cdot \sum_{j = 1}^m \sum_{k = 1}^d \cos \lprp{ \wt{h}_{j,k} (x - y)} - \E_{Z \thicksim \mc{N} (0, \norm*{x - y}^2)} \lsrs{\cos (Z)}} \leq \frac{\eps}{2}
    \end{equation*}
    with probability at least $1 - \delta / 2$.
\end{lemma}
\begin{proof}
    Let $f(x) = \cos (2 \diam (\mc{W})\cdot x)$. Note that $f$ is a $2\diam (\mc{W})$-Lipschitz function and hence, we get from \cref{thm:lip_conc} and our setting of $m$ that with probability at least $1 - \delta / 2$:
    \begin{equation*}
        \forall z \in \R^d \text{ s.t } \norm{z} \leq 1: \abs*{\frac{1}{md} \cdot \sum_{j = 1}^m \sum_{k = 1}^d f(\wt{h}_{j,k} (z)) - \E_{Z \thicksim \mc{N} (0, \norm{z}^2)} \lsrs{f(Z)}} \leq \frac{\eps}{2}.
    \end{equation*}
    The lemma now follows from the fact that for all $x, y \in \mc{W}$, $\norm{x - y} \leq 2 \diam(\mc{W})$. 
\end{proof}

The proof of \cref{thm:kern_approx} follows from \cref{lem:cos_dec_first,lem:cos_dec_second} and by noting:
\begin{equation*}
    \E_{Z \thicksim \mc{N} (0, \sigma^2)} [\cos Z] = 1 + \sum_{k = 1}^\infty \frac{(2k - 1)!!}{(2k)!} \cdot \sigma^{2k} = 1 + \sum_{k = 1}^\infty \frac{1}{k!} \cdot \lprp{\frac{\sigma^2}{2}}^k = \exp \lbrb{- \frac{\sigma^2}{2}}.
  \end{equation*}
\qed
\section{Distance Estimation}
\label{sec:dist_est}
In this section, we prove \cref{thm:de_thm}. First, in \cref{ssec:dist_est_alg}, we describe the data structure achieving the guarantees of \cref{thm:de_thm} and then prove its correctness in \cref{ssec:de_proof}. 

\subsection{Algorithm}
\label{ssec:dist_est_alg}

The pseudocode for our algorithm for distance estimation is defined in \cref{alg:dist_est_ds,alg:dist_est,alg:update_dist_est} with \cref{alg:dist_est_ds} instantiating the data structure with $0$ points in the dataset by only initializing the RHT. \cref{alg:dist_est,alg:update_dist_est} then outline the query and update procedures where
\begin{equation*}
    \forall r > 0: \psi_{r} (x) \coloneqq \min (\abs{x}, r).
\end{equation*}
Our query procedure is quite simple: we simply draw a small $\Ot (1)$ many random coordinates from $[md]$, $\lbrb{l_j}_{j = 1}^k$ and output the $\alpha$-quantile corresponding the entries $\{(y_{l_j} - (y_i)_{l_j})\}_{j \in [k]}$. If the distribution of the entries of $y - y_i$ were exactly Gaussian, the returned value would be exactly $\norm{q - x_i}$. In \cref{ssec:de_proof}, we simply invoke \cref{thm:lip_conc} and bound the incurred errors.

\begin{algorithm}[H]
    \begin{algorithmic}
        \State \textbf{Input: } Accuracy $\eps$, Failure probability $\delta$
        \State $m \gets \distEstRep$
        \State Let $\{D^i\}_{i = 1}^m$ be $m$ i.i.d random diagonal matrices with $D^{i}_{j,j} \overset{i.i.d}{\thicksim} \mc{N} (0, 1)$
        % \State For $i \in [n], j \in [m]$, let $y^j_i = H \cdot D^j x_i$
        % \State For $i \in [n]$, let $y_i = (y^1_i, \dots, y^m_i)$
        \State Let $\wt{h}^j (z) = HD^j z$ for all $j \in [m]$ and $\wt{h} (z) = (\wt{h}^1 (z), \wt{h}^2 (z), \dots, \wt{h}^m (z))$
        \State \textbf{Return: } $\lprp{\wt{h}, \phi}$
    \end{algorithmic}
    \caption{\prodDEDS}
    \label{alg:dist_est_ds}
  \end{algorithm}

\begin{algorithm}[H]
    \begin{algorithmic}
        \State \textbf{Input: } Data structure with point set $\lprp{\wt{h}, \{y_i\}_{i = 1}^n}$, Data point $x \in \R^d$, Failure probability $\delta$
        \State $y \gets \wt{h} (x)$
        \State Let $l_1, \dots, l_k$ be chosen uniformly from $[md]$ for $k = \distEstSamps$
        \State $\alpha \gets \Phi (3)$
        % \State Let $\alpha = \distEstQuant$ and let $\tau_i = \quant_{1 - \alpha} \lprp{\lbrb{\abs{y_{l_p} - (y_i)_{l_p}}}_{p = 1}^k}$ for all $i \in [n]$
        % \State Resample $l_1, \dots, l_k$ uniformly from $[q]$
        \State $r_i \gets 2 \sqrt{\log 1 / \eps} \cdot \quant_{\alpha} \lprp{\lbrb{y_{l_p} - (y_i)_{l_p}}_{p \in [k]}}$
        \State $d_i \gets \frac{1}{k} \cdot \sqrt{\frac{\pi}{2}} \cdot \sum_{p = 1}^k \psi_{r_i} (y_{l_p} - (y_i)_{l_p})$
        \State \textbf{Return: } $\{d_i\}_{i = 1}^n$
    \end{algorithmic}
    \caption{\prodDE}
    \label{alg:dist_est}
  \end{algorithm}

\begin{algorithm}[H]
    \begin{algorithmic}
        \State \textbf{Input: } Data structure with point set $\lprp{\wt{h}, \{y_i\}_{i = 1}^n}$, Data point to add $x_{n + 1} \in \R^d$
        \State $y_{n + 1} = \wt{h} (x_{n + 1})$
        \State \textbf{Return: } $\lbrb{\wt{h}, \{y_i\}_{i = 1}^{n + 1}}$
    \end{algorithmic}
    \caption{\updateDE}
    \label{alg:update_dist_est}
\end{algorithm}

\subsection{Proof of \cref{thm:de_thm}}
\label{ssec:de_proof}

The proof of \cref{thm:de_thm} will rely a sequence of applications of \cref{thm:lip_conc} applied to appropriately chosen Lipschitz functions outlined in the following claims.

\begin{claim}
    \label{clm:cdf_conc}
    Letting $\beta = \phi (4)$, we have:
    \begin{equation*}
        \forall z \text{ s.t } \norm{z} = 1: 2 \leq \quant_{\alpha - \beta / 4} \lprp{\lbrb{\wt{h}_{p,q} (z)}_{p \in [m], q \in [d]}} \leq \quant_{\alpha + \beta / 4} \lprp{\lbrb{\wt{h}_{p,q} (z)}_{p \in [m], q \in [d]}} \leq 4
    \end{equation*}
    with probability at least $1 - \delta / 4$.
\end{claim}
\begin{proof}
    We first define the functions:
    \begin{equation*}
        f(u) \coloneqq 
        \begin{cases}
            1 &  u \leq 2 \\
            (3 - u) & 2 \leq u \leq 3 \\
            0 & \text{otherwise}
        \end{cases}
        \text{ and }
        g(x) \coloneqq 
        \begin{cases}
            1 &  u \leq 3 \\
            (4 - u) & 3 \leq u \leq 4 \\
            0 & \text{otherwise}
        \end{cases}.
    \end{equation*}
    Note that both $f$ and $g$ are $1$-Lipschitz functions of $u$ and therefore, an application of \cref{thm:lip_conc} yields:
    \begin{gather*}
        \forall z \in \R^d \text{ s.t } \norm{z} \leq 1: \abs*{\frac{1}{md}\cdot \sum_{p = 1}^m \sum_{q = 1}^d f(\wt{h}_{p,q} (z)) - \E_{Z \thicksim \mathcal{N} (0, \norm{z}^2)} [f(Z)]} \leq \frac{\beta}{8} \\
        \forall z \in \R^d \text{ s.t } \norm{z} \leq 1: \abs*{\frac{1}{md}\cdot \sum_{p = 1}^m \sum_{q = 1}^d g(\wt{h}_{p,q} (z)) - \E_{Z \thicksim \mathcal{N} (0, \norm{z}^2)} [g(Z)]} \leq \frac{\beta}{8}
    \end{gather*}
    with probability at least $1 - \delta / 4$. Now, let $z \in \R^d$ with $\norm{z} = 1$. We now have:
    \begin{align*}
        \frac{1}{md} \cdot \sum_{p = 1}^m \sum_{q = 1}^d \bm{1} \lbrb{\wt{h}_{p,q} (z) \leq 2} &\leq \frac{1}{md} \cdot \sum_{p = 1}^m \sum_{q = 1}^d f(\wt{h}_{p,q} (z)) \leq \E_{Z \thicksim \mc{N}(0, 1)} [f(Z)] + \frac{\beta}{8} \\
        &= \alpha + \frac{\beta}{8} - \int_{2}^3 (u - 2) \phi(u) du \leq \alpha + \frac{\beta}{8} - \frac{\beta}{2} < \alpha - \frac{\beta}{4}
    \end{align*}
    yielding the first inequality in the conclusion of the claim. For the second, we have:
    \begin{align*}
        \frac{1}{md} \cdot \sum_{p = 1}^m \sum_{q = 1}^d \bm{1} \lbrb{\wt{h}_{p,q} (z) \leq 4} &\geq \frac{1}{md} \cdot \sum_{p = 1}^m \sum_{q = 1}^d g(\wt{h}_{p,q} (z)) \geq \E_{Z \thicksim \mc{N}(0, 1)} [g(Z)] - \frac{\beta}{8} \\
        &= \alpha - \frac{\beta}{8} + \int_{3}^4 (4 - u) \phi(u) du \geq \alpha - \frac{\beta}{8} + \frac{\beta}{2} > \alpha + \frac{\beta}{4}
    \end{align*}
    yielding the second and concluding the proof of the claim.
\end{proof}

\begin{claim}
    \label{clm:trunc_conc}
    We have:
    \begin{gather*}
        \forall z \text{ s.t } \norm{z} = 1, r \geq 4 \sqrt{\log 1 / \eps}: \lprp{1 - \frac{\eps}{2}} \leq \frac{1}{md} \cdot \sqrt{\frac{\pi}{2}} \cdot \sum_{p = 1}^m \sum_{q = 1}^d \psi_r (\wt{h}_{p,q} (z)) \leq \lprp{1 + \frac{\eps}{2}}
    \end{gather*}
    with probability at least $1 - \delta / 4$.
\end{claim}
\begin{proof}
    Note that the function $\sqrt{\pi / 2} \cdot \psi_r(u)$ is $\sqrt{\pi / 2}$-Lipschitz in $u$ for any $r$. Defining $r^* \coloneqq 4 \sqrt{\log 1 / \eps}$, two applications of \cref{thm:lip_conc} yield:
    \begin{gather*}
        \forall z \in \R^d \text{ s.t } \norm{z} \leq 1: \abs*{\frac{1}{md}\cdot \sqrt{\frac{\pi}{2}}\cdot \sum_{p = 1}^m \sum_{q = 1}^d \psi_{r^*}(\wt{h}_{p,q} (z)) - \sqrt{\frac{\pi}{2}} \cdot \E_{Z \thicksim \mathcal{N} (0, \norm{z}^2)} [\psi_{r^*} (Z)]} \leq \frac{\eps}{8} \\
        \forall z \in \R^d \text{ s.t } \norm{z} \leq 1: \abs*{\frac{1}{md}\cdot \sqrt{\frac{\pi}{2}} \cdot \sum_{p = 1}^m \sum_{q = 1}^d \abs{\wt{h}_{p,q} (z)} - \sqrt{\frac{\pi}{2}} \cdot \E_{Z \thicksim \mathcal{N} (0, \norm{z}^2)} [\abs{Z}]} \leq \frac{\eps}{8}
    \end{gather*}
    with probability at least $1 - \delta / 4$. Now, fix $z \in \R^d$ with $\norm{z} = 1$. We now have:
    \begin{gather*}
        \sqrt{\frac{\pi}{2}} \cdot \E_{Z \thicksim \mc{N} (0, 1)} \lsrs{\abs{Z}} = 1  \text{ and } \sqrt{\frac{\pi}{2}} \cdot \E_{Z \thicksim \mc{N} (0, 1)} \lsrs{\psi_{r^*} (Z)} \leq \sqrt{\frac{\pi}{2}} \cdot \E_{Z \thicksim \mc{N} (0, 1)} \lsrs{\abs{Z}} = 1.
    \end{gather*}
    We now derive a lower bound on $\E_{Z \thicksim \mc{N} (0, 1)} [\psi_{r^*} (Z)]$ below:
    \begin{align*}
        \sqrt{\frac{\pi}{2}} \cdot \E_{Z \thicksim \mc{N} (0, 1)} \lsrs{\psi_{r^*} (Z)} &\geq \sqrt{\frac{\pi}{2}} \cdot \E_{Z \thicksim \mc{N} (0, 1)} \lsrs{\abs{Z}} - \sqrt{\frac{\pi}{2}} \int_{(-\infty, -r^*] \cup [r^*, \infty)} \abs{u} \phi (u) du \\
        &= 1 - \sqrt{2\pi} \int_{r^*}^\infty u \phi(u) du = 1 - \exp \lbrb{- \frac{(r^*)^2}{2}} \geq 1 - \frac{\eps}{8}.
    \end{align*}
    The previous three displays now yield:
    \begin{gather*}
        \forall z \in \R^d \text{ s.t } \norm{z} = 1: \lprp{1 - \frac{\eps}{4}} \leq \frac{1}{md}\cdot \sqrt{\frac{\pi}{2}}\cdot \sum_{p = 1}^m \sum_{q = 1}^d \psi_{r^*}(\wt{h}_{p,q} (z)) \leq \lprp{1 + \frac{\eps}{4}} \\
        \forall z \in \R^d \text{ s.t } \norm{z} = 1: \lprp{1 - \frac{\eps}{4}} \leq \frac{1}{md}\cdot \sqrt{\frac{\pi}{2}}\cdot \sum_{p = 1}^m \sum_{q = 1}^d \abs{\wt{h}_{p,q} (z)} \leq \lprp{1 + \frac{\eps}{4}}
    \end{gather*}
    which imply the claim by noting that for any $r \geq r^*$, $\psi_{r^*} (x) \leq \psi_{r} (x) \leq \abs{x}$.
\end{proof}

For the rest of the proof, we will condition the conclusions of \cref{clm:cdf_conc,clm:trunc_conc}. We will now analyze the correctness of the query procedure. We first prove a correctness guarantee on estimated truncation levels $r_i$.

\begin{claim}
    \label{clm:trunc_est}
    Conditioned on \cref{clm:cdf_conc,clm:trunc_conc}, we have:
    \begin{equation*}
        \forall i \in [n]: 2 \norm{x - x_i} \leq \quant_\alpha \lprp{\lbrb{y_{l_p} - (y_i)_{l_p}}_{p \in [k]}} \leq 4 \norm{x - x_i}
    \end{equation*}
    with probability at least $1 - \delta / 4$ over the random indices $\{l_j\}_{j \in [k]}$.
\end{claim}
\begin{proof}
    Fix $i \in [n]$ and for all $j \in [k]$, define the random variables $W_j, V_j$:
    \begin{equation*}
        W_j \coloneqq \bm{1} \lbrb{(y_{l_j} - (y_i)_{l_j}) \leq 2 \norm{x - x_i}} \text{ and } V_j \coloneqq \bm{1} \lbrb{(y_{l_j} - (y_i)_{l_j}) \leq 4 \norm{x - x_i}}.
    \end{equation*}
    From the linearity of $\wt{h}$ and \cref{clm:cdf_conc}, we get:
    \begin{equation*}
        \E [W_j] \leq \alpha - \frac{\beta}{4} \text{ and } \E [V_j] \geq \alpha + \frac{\beta}{4}.
    \end{equation*}
    An application of Hoeffding's inequaity to the random variables $W_j, V_j$ with our setting of $k$ yields:
    \begin{equation*}
        \frac{1}{k}\cdot \sum_{j = 1}^k W_j \leq \alpha - \frac{\beta}{8} \text{ and } \frac{1}{k} \cdot \sum_{j = 1}^k V_j \geq \alpha + \frac{\beta}{8}
    \end{equation*}
    with probability at least $1 - \delta / (4n)$. On the above event, we have:
    \begin{equation*}
        2 \norm{x - x_i} \leq \quant_\alpha \lprp{\lbrb{y_{l_p} - (y_i)_{l_p}}_{p \in [k]}} \leq 4 \norm{x - x_i}.
    \end{equation*}
    A union bound over $i \in [n]$, concludes the proof of the claim.
\end{proof}

To conclude the proof, fix $i \in [n]$ and let $\wt{r} \coloneqq 4 \sqrt{\log 1 / \eps} \cdot \norm{x - x_i}$ and $\wh{r} \coloneqq 8 \sqrt{\log 1 / \eps} \cdot \norm{x - x_i}$. Noting that $\psi_r (u) \leq r$ for all $u \in \R, r > 0$, we have from Hoeffding's inequality, \cref{clm:trunc_conc} and our setting of $k$ that:
\begin{gather*}
    \frac{1}{k} \cdot \sqrt{\frac{\pi}{2}} \cdot \sum_{p = 1}^k \psi_{\wt{r}} ((y_{l_p} - (y_i)_{l_p})) \geq \lprp{1 - \eps} \cdot \norm{x - x_i} \\
    \frac{1}{k} \cdot \sqrt{\frac{\pi}{2}} \cdot \sum_{p = 1}^k \psi_{\wh{r}} ((y_{l_p} - (y_i)_{l_p})) \leq \lprp{1 + \eps} \cdot \norm{x - x_i}
\end{gather*}
with probability at least $1 - \delta / 4n$. A union bound yields the above condition for all $i \in [n]$ with probability at least $1 - \delta / 4$. Conditioning on the above display and \cref{clm:trunc_est,clm:trunc_conc,clm:cdf_conc} and noting that on this event $\psi_{\wt{r}} (u) \leq \psi_{r_i} (u) \leq \psi_{\wh{r}} (u)$ for all $u \in \R$, we get via a union bound:
\begin{equation*}
    \forall i \in [n]: (1 - \eps) \cdot \norm{x - x_i} \leq d_i \leq (1 + \eps) \cdot \norm{x - x_i}.
\end{equation*}
with probability at least $1 - \delta$. 

The runtime guarantee follow from the fact that for each $j \in [m]$, $\wt{h}^j(z)$ is computable in time $O(d\log d)$ for all $z \in \R^d$. This concludes the proof of the theorem up to routine runtime analyses.
\qed

\bibliographystyle{alpha}
\bibliography{adaptive}

\pagebreak
\appendix

\section{Optimality of \cref{thm:lip_conc}}
\label{sec:optimality}

In this section, we show that the conclusions of \cref{thm:lip_conc} are tight up to log factors and furthermore, that the embedding dimension required is within a logarithmic factor of that required by the use of full Gaussian matrices. We start by establishing the near optimality of conclusion of \cref{thm:lip_conc}:

\begin{theorem}
    \label{thm:m_lb}
    Let $\wt{h}, \wt{h}^j, \wt{h}_{j,k}$ be as in \ref{eq:RHT}, $\eps \in (0, 1)$ and $\delta \in (0, 0.01)$. Then, there exists a $1$-Lipschitz function, $f: \R \to \R$, such that:
    \begin{equation*}
        \exists z \in \R^d \text{ s.t } \norm{z} \leq 1: \abs*{\frac{1}{md} \cdot \sum_{j = 1}^m \sum_{k = 1}^d f(\wt{h}_{j,k} (z)) - \E_{Z \thicksim \mc{N} (0, \norm{z}^2)} \lsrs{f(Z)}} \geq \eps
    \end{equation*}
    with probability at least $\delta$ if $m \leq \eps^{-2} \log d / \delta$.
\end{theorem}
\begin{proof}
    In our construction, we simply let $f(x) \coloneqq x$ and we pick $z \coloneqq e_i$, the basis vectors. Now, we see:
    \begin{align*}
        \frac{1}{md} \cdot \sum_{j = 1}^m \sum_{k = 1}^d f(\wt{h}_{j,k} (e_i)) &= \frac{1}{m} \cdot \sum_{j = 1}^m D^j_{i,i} \eqqcolon W_i \overset{iid}{\thicksim} \mc{N} (0, 1 / m).
    \end{align*}
    By noting that for $Z \thicksim \mc{N} (0, 1)$, $\P \lbrb{\abs{Z} \geq x} \geq \frac{1}{2} \cdot \lprp{\frac{1}{x} - \frac{1}{x^3}} \cdot \exp \lprp{- \frac{x^2}{2}}$, we get:
    \begin{equation*}
        \P \lbrb{\abs*{W_i} \geq \eps} \geq 1.5 \cdot \frac{\delta}{d} \implies \P \lbrb{\exists i \in [d]: \abs*{W_i} \geq \eps} \geq 1 - \lprp{1 - 1.5 \cdot \frac{\delta}{d}}^d \geq \delta.
    \end{equation*}
    The theorem now follows from the observation that $\E_{Z \thicksim \mc{N} (0, \sigma^2)} [f(Z)] = 0$ for any $\sigma^2$.
\end{proof}

We now show that the embedding dimension guaranteed by \cref{thm:lip_conc} are within a logarithmic factor of the best obtainable dimension even in the setting where true Gaussian matrices are used. 

\begin{theorem}
    \label{thm:gau_lb}
    Let $\{g_i\}_{i = 1}^n$ be $n$ i.i.d random vectors such that $g_i \thicksim \mc{N} (0, I)$, $\eps \in (0, 1)$ and $d \geq 40$. Then, there exists a $1$-Lipschitz function, $f: \R \to \R$, such that:
    \begin{equation*}
        \exists z \in \R^d \text{ s.t } \norm{z} \leq 1: \abs*{\frac{1}{n} \sum_{i = 1}^n f(\inp{z}{g_i}) - \E_{Z \thicksim \mc{N} (0, \norm{z}^2)} [f(Z)]} \geq \eps
    \end{equation*}
    with probability at least $9 / 10$ as long as $n \leq (2\eps)^{-2} d$.
\end{theorem}
\begin{proof}
    As in \cref{thm:m_lb}, we let $f(x) \coloneqq x$. Observe that we now have:
    \begin{equation*}
        \max_{z \text{ s.t } \norm{z} \leq 1} \frac{1}{n} \sum_{i = 1}^n f(\inp{z}{g_i}) = \max_{z \text{ s.t } \norm{z} \leq 1} \inp*{z}{\frac{1}{n} \cdot \sum_{i = 1}^n g_i} = \norm*{\frac{1}{n} \cdot \sum_{i = 1}^n g_i}.
    \end{equation*}
    Letting $g \coloneqq \frac{1}{n} \cdot \sum_{i = 1}^n g_i$, we see that $g \thicksim \mc{N} (0, I / n)$ and hence, we get from \cref{thm:tsirelson} that:
    \begin{equation*}
        \P \lbrb{\norm*{\frac{1}{n} \cdot \sum_{i = 1}^n g_i} \geq \eps} \geq \frac{9}{10}.
    \end{equation*}
    The theorem now follows from the fact that $\E_{Z \thicksim \mc{N} (0, \sigma^2)} [f(Z)] = 0$ for all $\sigma^2 > 0$.
\end{proof}
\section{Miscellaneous Results and Supporting Lemmas}
\label{sec:misc}

In this section, we recall some standard facts from probability theory and some lemmas used in the proofs our main results. The first is the Tsirelson-Ibragimov-Sudakov inequality (see \cite[Theorem 5.6]{blm}, for example):

\begin{theorem}{\cite[Theorem 5.6]{blm}}
    \label{thm:tsirelson}
    Let $X = (X_1, \dots, X_n)$ be a vector of $n$ independent standard normal random variables. Let $f: \mb{R}^n \to \mb{R}$ be an $L$-Lipschitz function; that is, $f$ satisfies:
    \begin{equation*}
        \forall x, y \in \R^d: \abs{f(x) - f(y)} \leq L \cdot \norm{x - y}.
    \end{equation*}
    Then, we have:
    \begin{equation*}
        \P \lbrb{f(X) - \E f(X) \geq t} \leq e^{-t^2 / (2L^2)}.
    \end{equation*}
\end{theorem}
We have the following simple corollary.
\begin{corollary}
    \label{cor:tsirelson_two_sided}
    Let $X = (X_1, \dots, X_n)$ be a vector of $n$ independent standard normal random variables. Let $f: \mb{R}^n \to \mb{R}$ be an $L$-Lipschitz function; that is, $f$ satisfies:
    \begin{equation*}
        \forall x, y \in \R^d: \abs{f(x) - f(y)} \leq L \cdot \norm{x - y}.
    \end{equation*}
    Then, we have:
    \begin{equation*}
        \P \lbrb{\abs{f(X) - \E f(X)} \geq t} \leq 2e^{-t^2 / (2L^2)}.
    \end{equation*}
\end{corollary}
\begin{proof}
    The proof follows by applying \cref{thm:tsirelson} to $-f$ and a union bound.
\end{proof}

We now recall the definition of an $\eps$-net and a Covering number from \cite{vershynin}. 
\begin{definition}{\cite[Definition 4.2.1]{vershynin}}
    \label{def:eps_net}
    Let $(T, d)$ be a metric space. Consider a subset $K \subset T$ and let $\eps > 0$. A subset $\mc{N} \subseteq K$ is called an \emph{$\eps$-net} of $K$ if every point in $K$ is within a distance $\eps$ of some point in $\mc{N}$, i.e
    \begin{equation*}
        \forall x \in K\ \exists x_0 \in \mc{N}: d(x, x_0) \leq \eps.
    \end{equation*}

    Additionally, the smallest possible cardinality of an $\eps$-net of $K$ is called a \emph{covering number} of $K$ and is denoted $\mc{N} (K, d, \eps)$. 
\end{definition}
We now introduce a standard proposition bounding covering numbers of subsets of Euclidean space.
\begin{proposition}{\cite[Proposition 4.2.12]{vershynin}}
    \label{prop:cov_bnd}
    Let $K$ be a subset of $\R^d$ and let $\eps > 0$. Then:
    \begin{equation*}
        \frac{\abs{K}}{\abs{\eps B_2^n}} \leq \mc{N} (K, \norm{\cdot}_2, \eps) \leq \frac{\abs{K + (\eps / 2) B_2^n}}{\abs{(\eps / 2) B_2^n}}.
    \end{equation*}
    Here, $\abs{\cdot}$ denotes the volume in $\R^n$, $B_2^n$ denotes the unit Euclidean ball in $\R^n$; so, $\eps B_2^n$ is an Euclidean ball with radius $\eps$. 
\end{proposition}
The proposition yields the following simple corollary similar to \cite[Corollary 4.2.13]{vershynin}.
\begin{corollary}
    \label{lem:cov_bnd}
    Assume the setting of \cref{prop:cov_bnd}. Additionally, let $K$ be a subset of $B^n_2$. Then, 
    \begin{equation*}
        \mc{N} (K, \norm{\cdot}_2, \eps) \leq \lprp{\frac{2}{\eps} + 1}^{n}.
    \end{equation*}
\end{corollary}
\begin{proof}
    We have from \cref{prop:cov_bnd}:
    \begin{equation*}
        \mc{N} (K, \norm{\cdot}_2, \eps) \leq \frac{\abs*{K + (\eps / 2) B^n_2}}{\abs*{(\eps / 2) B^n_2}} \leq \frac{\abs*{(1 + \eps / 2) B^n_2}}{\abs*{(\eps / 2) B^n_2}} = \frac{(1 + \eps / 2)^n}{(\eps / 2)^n} = \lprp{\frac{2}{\eps} + 1}^n.
    \end{equation*}
\end{proof}

We will also use a simple lemma concerning the spectral norm of the linear transformations considered throughout this paper. 

\begin{lemma}
    \label{lem:h_spec_bnd}
    Let $\eps, \delta \in (0, 1/2)$ and $d \in \N$. Then, we have:
    \begin{equation*}
        \forall x,y \in \R^{d}: (1 - \eps) \cdot \norm{x - y} \leq \frac{\norm*{\wt{h} (x) - \wt{h} (y)}}{\sqrt{md}} \leq (1 + \eps) \cdot \norm{x - y}
    \end{equation*}
    with probability at least $1 - \delta$ if $m \geq 4 \cdot \frac{\log d + \log 2 / \delta}{\eps^2}$. 
\end{lemma}
\begin{proof}
    First, note that we may assume $y = 0$ due to the linearity of $\wt{h}$ and that it now suffices to show the conclusion for $x \in \S^{d - 1}$. We have from the orthogonality of $H$:
    \begin{align*}
        \norm*{\wt{h} (x)}^2 &= (\wt{h} (x))^\top \wt{h} (x) = d x^\top \sum_{j = 1}^m (D^j)^2 x = d x^\top 
        \begin{bmatrix}
            \sum_{j = 1}^m (D^j_{1,1})^2 & 0 & \cdots & 0 \\
            0 & \sum_{j = 1}^m (D^j_{2,2})^2 & \cdots & 0 \\
            \vdots & \vdots & \ddots & \vdots \\
            0 & 0 & \cdots & \sum_{j = 1}^m (D^j_{d,d})^2
        \end{bmatrix}
        x.
    \end{align*}
    Therefore, we have:
    \begin{equation*}
        \forall x \in \R^d: \min_{k} \sqrt{\frac{\sum_{j = 1}^m (D^j)_{k,k}^2}{m}} \leq \frac{\norm{\wt{h} (x)}}{\sqrt{md}} \leq \max_{k} \sqrt{\frac{\sum_{j = 1}^m (D^j)_{k,k}^2}{m}}.
    \end{equation*}
    Note that we have from \cref{cor:tsirelson_two_sided} that:
    \begin{equation*}
        \abs*{\sqrt{\frac{\sum_{j = 1}^m (D^j)_{k,k}^2}{m}} - 1} \leq \sqrt{\frac{2\cdot (\log d + \log 2 / \delta)}{m}}
    \end{equation*}
    with probability at least $1 - \delta / d$. A union bound over $k \in [d]$ and our condition on $m$ concludes the proof. 
\end{proof}

% We also recall the classical Hoeffding and Bernstein inequalities:
% \begin{theorem}{\cite[Theorem 2.8]{blm}}
%     \label{thm:hoeffding}
%     Let $X_1, \dots, X_n$ be independent real valued random variables such that $X_i$ takes its values in $[a_i, b_i]$ almost surely for all $i \leq n$. Then, we have:
%     \begin{equation*}
%         \forall t \geq 0: \P \lbrb{\sum_{i = 1}^n X_i - \E [X_i] \geq t} \leq \exp \lprp{- \frac{2t^2}{\sum_{i = 1}^n (a_i - b_i)^2}}.
%     \end{equation*}
% \end{theorem}

% \begin{theorem}
%     \label{thm:bernstein}{\cite[Corollary 2.11]{blm}}
%     Let $X_1, \dots, X_n$ be independent real valued random variables. Assume that there exist positive numbers, $\nu$ and $c$ such that:
%     \begin{equation*}
%         \sum_{i = 1}^n \E \lsrs{X_i^2} \leq \nu \text{ and } \forall \text{ integers } q \geq 3: \sum_{i = 1}^n \E \lsrs{\abs{X_i}^q} \leq \frac{q!}{2} \nu c^{q - 2}.
%     \end{equation*}
%     Then, we have:
%     \begin{equation*}
%         \forall t \geq 0: \P \lbrb{\sum_{i = 1}^n (X_i  - \E [X_i]) \geq t} \leq \exp \lbrb{- \frac{t^2}{2(\nu + ct)}}.
%     \end{equation*}
% \end{theorem}

\end{document}